\documentclass{article}

\PassOptionsToPackage{numbers,sort}{natbib}

\usepackage[utf8]{inputenc} 
\usepackage[T1]{fontenc}    
\usepackage[hyperfootnotes=false]{hyperref}       
\usepackage{url}            
\usepackage{booktabs}       
\usepackage{amsfonts}       
\usepackage{nicefrac}       
\usepackage{microtype}      
\usepackage{enumitem}
\usepackage{graphicx}
\usepackage{hhline}
\usepackage{algorithm}
\usepackage[table]{xcolor}
\usepackage{tikzscale}
\usepackage{standalone}
\usepackage{algpseudocode}
\usepackage{comment}
\usepackage{algorithmicx}
\usepackage[font={small}]{caption}
\usepackage{subcaption}
\usepackage[most]{tcolorbox}

\usepackage{tabu}
\usepackage{tikz}

\usepackage{amsmath,amssymb,amsthm}
\usepackage{bm}

\usepackage{xcolor}
\usepackage{etoolbox}

\newtoggle{arxiv}
\toggletrue{arxiv}

\iftoggle{arxiv}{
  \usepackage{authblk}
  \usepackage[numbers]{natbib}
  \setlength{\textwidth}{6.5in}
  \setlength{\textheight}{9in}
  \setlength{\oddsidemargin}{0in}
  \setlength{\evensidemargin}{0in}
  \setlength{\topmargin}{-0.5in}
  \newlength{\defbaselineskip}
  \setlength{\defbaselineskip}{\baselineskip}
  \setlength{\marginparwidth}{0.8in}
}{
  \PassOptionsToPackage{numbers}{natbib}
  \usepackage[final]{neurips_2018}
}

\newtheorem{theorem}{Theorem}
\newtheorem{proposition}{Proposition}
\newtheorem{lemma}{Lemma}
\newtheorem{corollary}{Corollary}
\theoremstyle{definition}

\newcommand{\R}{\mathbb{R}}
\DeclareMathOperator{\diag}{diag}
\DeclareMathOperator{\rank}{rank}
\newcommand{\defeq}{:=}
\providecommand{\sign}{\mathop{\rm sign}}
\providecommand{\abs}[1]{\left\lvert#1\right\rvert}
\newcommand\VCdim{{\operatorname{VCdim}}}
\DeclareMathOperator{\vect}{vec}

\newcommand{\vzero}{\mathbf{0}}
\newcommand{\vA}{\mathbf{A}}
\newcommand{\vB}{\mathbf{B}}
\newcommand{\vC}{\mathbf{C}}
\newcommand{\vD}{\mathbf{D}}
\newcommand{\vF}{\mathbf{F}}
\newcommand{\vG}{\mathbf{G}}
\newcommand{\vH}{\mathbf{H}}
\newcommand{\vI}{\mathbf{I}}
\newcommand{\vM}{\mathbf{M}}
\newcommand{\vN}{\mathbf{N}}

\newcommand{\vR}{\mathbf{R}}
\newcommand{\vS}{\mathbf{S}}
\newcommand{\vT}{\mathbf{T}}
\newcommand{\vZ}{\mathbf{Z}}
\newcommand{\vb}{\mathbf{b}}
\newcommand{\ve}{\mathbf{e}}
\newcommand{\vg}{\mathbf{g}}
\newcommand{\vh}{\mathbf{h}}
\newcommand{\vi}{\mathbf{i}}
\newcommand{\vs}{\mathbf{s}}
\newcommand{\vu}{\mathbf{u}}
\newcommand{\vv}{\mathbf{v}}
\newcommand{\vx}{\mathbf{x}}
\newcommand{\vy}{\mathbf{y}}

\newcommand{\cD}{\mathcal{D}}
\newcommand{\cF}{\mathcal{F}}
\newcommand{\cK}{\mathcal{K}}
\newcommand{\cS}{\mathcal{S}}

\newcommand{\LDRSD}{LDR-SD}
\newcommand{\LDRTD}{LDR-TD}
\newcommand*\samethanks[1][\value{footnote}]{\footnotemark[#1]}

\title{Learning Compressed Transforms \\ with Low Displacement Rank}

\iftoggle{arxiv}{
\author[$\dagger$]{Anna T. Thomas\thanks{These authors contributed equally.}}
\author[$\dagger$]{Albert Gu\samethanks}
\author[$\dagger$]{Tri Dao}
\author[$\ddagger$]{Atri Rudra}
\author[$\dagger$]{Christopher R{\'e}}
\affil[$\dagger$]{Department of Computer Science, Stanford University}
\affil[$\ddagger$]{Department of Computer Science and Engineering, University at Buffalo, SUNY\vspace{4pt}}
\affil[ ]{{\texttt{\{thomasat,albertgu,trid\}@stanford.edu}, \texttt{atri@buffalo.edu}, \texttt{chrismre@cs.stanford.edu}}}
}
{
\author{
Anna T. Thomas${^\dagger}$\thanks{These authors contributed equally.}, Albert Gu${^\dagger}$\samethanks, Tri Dao$^\dagger$, Atri Rudra${^\ddagger}$, Christopher R{\'e}$^{\dagger}$\\
$^\dagger$ Department of Computer Science, Stanford University\\
$^\ddagger$ Department of Computer Science and Engineering, University at Buffalo, SUNY\\
\footnotesize{\texttt{\{thomasat,albertgu,trid\}@stanford.edu}, \texttt{atri@buffalo.edu}, \texttt{chrismre@cs.stanford.edu}} \\
}}\begin{document}

\maketitle

\begin{abstract}
  The low displacement rank (LDR) framework for structured matrices represents a matrix through two displacement operators and a low-rank residual. Existing use of LDR matrices in deep learning has applied fixed displacement operators encoding forms of shift invariance akin to convolutions. We introduce a class of LDR matrices with more general displacement operators, and explicitly learn over both the operators and the low-rank component. This class generalizes several previous constructions while preserving compression and efficient computation. We prove bounds on the VC dimension of multi-layer neural networks with structured weight matrices and show empirically that our compact parameterization can reduce the sample complexity of learning. 
  When replacing weight layers in fully-connected, convolutional, and recurrent neural networks for image classification and language modeling tasks, our new classes
  exceed the accuracy of existing compression approaches, and on some tasks also outperform general unstructured layers while using more than 20x fewer parameters.

\end{abstract}

\section{Introduction}

Recent years have seen a surge of interest in structured representations for deep learning, motivated by achieving compression and acceleration while maintaining generalization properties.
A popular approach for learning compact models 
involves constraining the weight matrices to exhibit
some form of dense but compressible structure and learning directly over the parameterization of
this structure.
Examples of structures explored for the weight matrices of deep learning pipelines include low-rank matrices~\cite{denil2013predicting,sainath2013low},
low-distortion projections~\cite{yang2015deep},
(block-)circulant matrices~\cite{cheng2015exploration,ding2017circnn},
Toeplitz-like matrices~\cite{lu2016learning,sindhwani2015structured},
and constructions derived from Fourier-related transforms~\cite{moczulski2015acdc}.
Though they confer significant storage and computation benefits, these constructions tend to underperform general fully-connected layers in deep learning.
This raises the question of whether broader classes of structured matrices can achieve superior downstream performance while retaining compression guarantees.

Our approach leverages the \textbf{low displacement rank} (LDR) framework (Section~\ref{sec:DR-background}), which encodes structure through two sparse \emph{displacement operators} and a low-rank residual term~\cite{kailath1979displacement}.
Previous work studying neural networks with LDR weight matrices assumes fixed displacement operators and learns only over the residual~\cite{sindhwani2015structured,zhao2017theoretical}.
The only case attempted in practice that explicitly employs the LDR framework uses fixed operators encoding shift invariance,
producing weight matrices which were found to achieve superior downstream quality than several other compression approaches~\cite{sindhwani2015structured}.
Unlike previous work, we consider learning the displacement operators \emph{jointly} with the low-rank residual.
Building upon recent progress on structured dense matrix-vector multiplication~\cite{desa2018two}, we introduce a more general class of LDR matrices and develop practical algorithms for using these matrices in deep learning architectures.
We show that the resulting class of matrices subsumes many previously used structured layers, including constructions that did not explicitly use the LDR framework~\cite{moczulski2015acdc,ding2017circnn}.
When compressing weight matrices in fully-connected, convolutional, and recurrent neural networks, we empirically demonstrate improved accuracy over existing approaches.
Furthermore, on several tasks our constructions achieve higher accuracy than general unstructured layers while using an order of magnitude fewer parameters.

To shed light on the empirical success of LDR matrices in machine learning, we draw connections to recent work on learning equivariant representations, and hope to motivate further investigations of this link.
Notably, many successful previous methods for compression apply classes of structured matrices related to convolutions~\cite{cheng2015exploration,ding2017circnn,sindhwani2015structured};
while their explicit aim is to accelerate training and reduce memory costs,
this constraint implicitly encodes a shift-invariant structure that is well-suited for image and audio data.
We observe that the LDR construction enforces a natural notion of approximate equivariance to transformations governed by the displacement operators, suggesting that, in contrast, our approach of learning the operators allows for modeling and learning more general latent structures in data that may not be precisely known in advance.

Despite their increased expressiveness, our new classes retain the storage and computational benefits of conventional structured representations. 
Our construction provides guaranteed compression (from quadratic to linear parameters) and matrix-vector multiplication algorithms that are quasi-linear in the number of parameters.
We additionally provide the first analysis of the sample complexity of learning neural networks with LDR weight matrices, which extends to low-rank, Toeplitz-like and other previously explored fixed classes of LDR matrices. More generally, our analysis applies to structured matrices whose parameters can interact multiplicatively with high degree.
We prove that the class of neural networks constructed from these matrices retains VC dimension almost linear in the number of parameters, which implies that LDR matrices with learned displacement operators are still efficiently recoverable from data.
This is consistent with our empirical results, which suggest that constraining weight layers to our broad class of LDR matrices can reduce the sample complexity of learning compared to unstructured weights.

We provide a detailed review of previous work and connections to our approach in Appendix~\ref{sec:related-work}.

\paragraph{Summary of contributions:}
\begin{itemize}
  \item We introduce a rich class of LDR matrices where the displacement operators are explicitly learned from data, and provide multiplication algorithms implemented in PyTorch (Section~\ref{sec:our-approach}).\footnote{Our code is available at \url{https://github.com/HazyResearch/structured-nets}.}
\item We prove that the VC dimension of multi-layer neural networks with LDR weight matrices, which encompasses a broad class of previously explored approaches including the low-rank and Toeplitz-like classes, is quasi-linear in the number of parameters~(Section~\ref{sec:theory}). 
\item We empirically demonstrate that our construction improves downstream quality when compressing weight layers in fully-connected, convolutional, and recurrent neural networks compared to previous compression approaches, and on some tasks can even outperform general unstructured layers (Section~\ref{sec:eval}).

\end{itemize}

\section{Background: displacement rank}
\label{sec:DR-background}

The generic term \emph{structured matrix} refers to an \(m \times n\) matrix that can be represented in much fewer than \(mn\) parameters, and admits fast operations such as matrix-vector multiplication.
The displacement rank approach represents a structured matrix \(\mathbf{M} \in \R^{m \times n}\) through \textbf{displacement operators} \((\mathbf{A}\in\R^{m \times m},\mathbf{B}\in\R^{n \times n})\) defining a linear map \(\nabla_{\vA,\vB} : \mathbf{M} \mapsto \mathbf{A}\mathbf{M}-\mathbf{M}\mathbf{B}\) on matrices, and a \textbf{residual} \(\mathbf{R}\), so that if
\begin{equation}
  \label{eq:DR}
  \mathbf{A}\mathbf{M}-\mathbf{M}\mathbf{B}=\mathbf{R}
\end{equation}
then \(\mathbf{M}\) can be manipulated solely through the compressed representation \((\mathbf{A},\mathbf{B},\mathbf{R})\).  We assume that $\vA$ and $\vB$ have disjoint eigenvalues, which guarantees that $\vM$ can be recovered from $\vA,\vB,\vR$ (c.f.\ Theorem 4.3.2, \citet{pan2012structured}). 
The rank of $\mathbf{R}$ (also denoted $\nabla_{\vA,\vB}[\vM]$) is called the \textbf{displacement rank} of $\vM$ w.r.t. $(\vA,\vB)$.\footnote{Throughout this paper, we use square matrices for simplicity, but LDR is well-defined for rectangular.}

The displacement approach was originally introduced to describe the \emph{Toeplitz-like} matrices, which are not perfectly Toeplitz but still have shift-invariant structure~\cite{kailath1979displacement}.
These matrices have LDR with respect to \emph{shift/cycle} operators.
A standard formulation uses $\vA = \vZ_1, \vB = \vZ_{-1}$, where $\vZ_f = \begin{bmatrix} 0_{1 \times (n-1)} & f\\ \vI_{n-1} & 0_{(n-1) \times 1} \end{bmatrix}$ denotes the matrix with $1$ on the subdiagonal and $f$ in the top-right corner. The Toeplitz-like matrices have previously been applied in deep learning and kernel approximation, and in several cases have performed significantly better than competing compressed approaches~\cite{sindhwani2015structured,lu2016learning,choromanski2016recycling}. Figure~\ref{fig:toeplitz} illustrates the displacement~\eqref{eq:DR} for a Toeplitz matrix, showing how the shift invariant structure of the matrix leads to a residual of rank at most 2.

\begin{figure}[th]
  \centering
  \includegraphics[width=\linewidth]{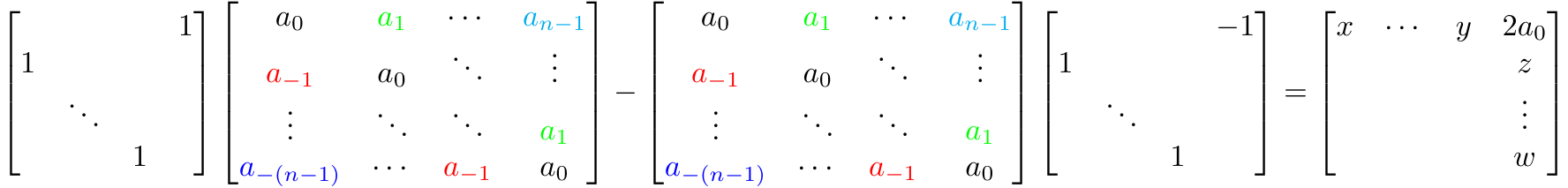}
  \caption{Displacement equation for a Toeplitz matrix with respect to shift operators $\vZ_1, \vZ_{-1}$.} 
  \label{fig:toeplitz}
\end{figure}

A few distinct classes of useful matrices are known to satisfy a displacement property: the classic types are the Toeplitz-, Hankel-, Vandermonde-, and Cauchy-like matrices (Appendix~\ref{sec:DR-proofs}, Table~\ref{table:displacements}), which are ubiquitous in other disciplines~\cite{pan2012structured}.
These classes have fixed operators consisting of diagonal or shift matrices, and LDR properties have traditionally been analyzed in detail only for these special cases.
Nonetheless, a few elegant properties hold for generic operators, stating that certain combinations of (and operations on) LDR matrices preserve low displacement rank.
We call these \emph{closure properties}, and introduce an additional block closure property that is related to convolutional filter channels (Section~\ref{sec:conv}).

We use the notation $\cD_{\vA,\vB}^r$ to refer to the matrices of displacement rank $\le r$ with respect to $(\vA,\vB)$.
\begin{proposition}
  \label{prop:closure}
  LDR matrices are closed under the following operations:

  \begin{enumerate}[label=(\alph*)]
    \item\label{prop:closure:transpose-inverse} \textbf{Transpose/Inverse} If $\vM \in \cD_{\vA,\vB}^r$, then $\vM^T \in \cD_{\vB^T,\vA^T}^r$ and $\vM^{-1} \in \cD_{\vB,\vA}^r$.

    \item\label{prop:closure:sum} \textbf{Sum} If $\vM \in \cD_{\vA,\vB}^r$ and $\vN \in \cD_{\vA,\vB}^s$, then $\vM+\vN \in \cD_{\vA,\vB}^{r+s}$.
    \item\label{prop:closure:product} \textbf{Product} If $\vM \in \cD_{\vA,\vB}^r$ and $\vN \in \cD_{\vB,\vC}^s$, then $\vM\vN \in \cD_{\vA,\vC}^{r+s}$.
    \item\label{prop:closure:block} \textbf{Block}
    Let $\vM_{ij}$ satisfy $\vM_{ij} \in \cD_{\vA_i,\vB_j}^r$ for $i = 1\dots k, j = 1 \dots \ell$.
    Then the $k \times \ell$ block matrix $(\vM_{ij})_{ij}$ has displacement rank $rk\ell$.

  \end{enumerate}
\end{proposition}
Proposition~\ref{prop:closure} is proved in Appendix~\ref{sec:DR-proofs}.

\section{Learning displacement operators}
\label{sec:our-approach}

We consider two classes of new displacement operators.
These operators are fixed to be matrices with particular sparsity patterns, where the entries are treated as learnable parameters.

The first operator class consists of \textbf{subdiagonal} (plus corner) matrices: $\vA_{i+1,i}$, along with the corner $\vA_{0,n-1}$, are the only possible non-zero entries. As $\vZ_f$ is a special case matching this sparsity pattern, this class is the most direct generalization of Toeplitz-like matrices with learnable operators.

The second class of operators are \textbf{tridiagonal} (plus corner) matrices: with the exception of the outer corners $\vA_{0,n-1}$ and $\vA_{n-1,0}$, $\vA_{i,j}$ can only be non-zero if $|i-j| \le 1$. Figure~\ref{fig:operator_matrices} shows the displacement operators for the Toeplitz-like class and our more general operators.
We henceforth let \LDRSD{} and \LDRTD{} denote the classes of matrices with low displacement rank with respect to subdiagonal and tridiagonal operators, respectively. Note that \LDRTD{} contains \LDRSD{}.

\begin{figure}[th]
  \centering
  \begin{alignat*}{3}
      \begin{bmatrix}
        0 & & \cdots & 0 & f \\
        1 & 0 & & \ddots & 0 \\
        \vdots & 1 & \ddots & & \vdots \\
        0 & \ddots & \ddots & \ddots &  \\
        0 & 0 & \dots & 1 & 0
      \end{bmatrix}
      \quad
  &
      \quad
      \begin{bmatrix}
        0 & & \cdots & 0 & x_0 \\
        x_1 & 0 & & \ddots & 0 \\
        \vdots & x_2 & \ddots & & \vdots \\
        0 & \ddots & \ddots & \ddots &  \\
        0 & 0 & \dots & x_{n-1} & 0
      \end{bmatrix}
      \quad
          &
      \quad
      \begin{bmatrix}
        b_0 & a_0 & \cdots & 0 & s \\
        c_0 & b_1 & a_1 & & 0 \\
        \vdots & c_1 & \ddots & \ddots & \vdots \\
        0 & & \ddots & b_{n-1} & a_{n-2} \\
        t & 0 & \dots & c_{n-2} & b_{n-1}
      \end{bmatrix}
  \end{alignat*}
  \caption{The $\vZ_f$ operator (left), and our learnable subdiagonal (center) and tridiagonal (right) operators, corresponding to our proposed LDR-SD and LDR-TD classes.}
  \label{fig:operator_matrices}
\end{figure}

\paragraph{Expressiveness}
The matrices we introduce can model rich structure and subsume many types of linear transformations used in machine learning.
We list some of the structured matrices that have LDR with respect to tridiagonal displacement operators:
\begin{proposition}
  \label{prop:richness}
  The LDR-TD matrices contain:

  \begin{enumerate}[label=(\alph*)]
    \item Toeplitz-like matrices, which themselves include many Toeplitz and circulant variants (including standard convolutional filters - see Section~\ref{sec:conv} and Appendix~\ref{sec:DR-proofs}, Corollary~\ref{cor:toeplitz})~\cite{sindhwani2015structured,cheng2015exploration,ding2017circnn}.
     \item low-rank matrices.
     \item the other classic displacement structures: Hankel-like, Vandermonde-like, and Cauchy-like matrices.
     \item orthogonal polynomial transforms, including the Discrete Fourier and Cosine Transforms.
     \item combinations and derivatives of these classes via the closure properties (Proposition~\ref{prop:closure}), including structured classes previously used in machine learning such as ACDC~\cite{moczulski2015acdc} and block circulant layers~\cite{ding2017circnn}.
   \end{enumerate}

\end{proposition}

These reductions are stated more formally and proved in Appendix~\ref{sec:expressiveness-proof}. We also include a diagram of the structured matrix classes included by the proposed LDR-TD class in Figure~\ref{fig:expressivity} in Appendix~\ref{sec:expressiveness-proof}.

\paragraph{Our parameterization}
Given the parameters $\vA,\vB,\vR$, the operation that must ultimately be performed is matrix-vector multiplication by $\vM = \nabla_{\vA,\vB}^{-1}[\vR]$.
Several schemes for explicitly reconstructing $\vM$ from its displacement parameters are known for specific cases~\cite{pan2003inversion,simoncini2016computational}, but do not always apply to our general operators.
Instead, we use $\vA,\vB,\vR$ to implicitly construct a slightly different matrix with at most double the displacement rank, which is simpler to work with.
\begin{proposition}
  \label{prop:reconstruction}
  Let $\cK(\vA,\vv)$ denote the $n \times n$ \emph{Krylov matrix}, defined to have $i$-th column $\vA^i\vv$.
  For any vectors $\vg_1, \dots, \vg_r, \vh_1, \dots, \vh_r \in \R^n$, then the matrix
  \begin{equation}
    \label{eq:reconstruction}
    \sum_{i=1}^r \cK(\vA,\vg_i)\cK(\vB^T,\vh_i)^T
  \end{equation}
  has displacement rank at most $2r$ with respect to $\vA^{-1},\vB$.
\end{proposition}

Thus our representation stores the parameters $\vA,\vB,\vG,\vH$, where $\vA,\vB$ are either subdiagonal or tridiagonal operators (containing $n$ or $3n$ parameters), and $\vG,\vH \in \R^{n \times r}$.
These parameters implicitly define the matrix~\eqref{eq:reconstruction}, which is the LDR weight layer we use.

\paragraph{Algorithms for LDR-SD}
Generic and near-linear time algorithms for matrix-vector multiplication by LDR matrices with even more general operators, including both the LDR-TD and LDR-SD classes, were recently shown to exist~\cite{desa2018two}.
However, complete algorithms were not provided, as they relied on theoretical results such as the transposition principle~\cite{transposition} that only imply the existence of algorithms.
Additionally, the recursive polynomial-based algorithms are difficult to implement efficiently. 
For LDR-SD, we provide explicit and complete near-linear time algorithms for multiplication by~\eqref{eq:reconstruction}, as well as substantially simplify them to be useful in practical settings and implementable with standard library operations.
We empirically compare the efficiency of our implementation and unstructured matrix-vector multiplication in Figure~\ref{fig:speed} and Table~\ref{table:speed} in Appendix~\ref{sec:additional-results}, showing that LDR-SD accelerates inference by 3.34-46.06x for $n \geq 4096$.
We also show results for the low-rank and Toeplitz-like classes, which have a lower computational cost.
For LDR-TD, we explicitly construct the $\cK(\vA,\vg_i)$ and $\cK(\vB^T,\vh_i)$ matrices for $i=1,...,r$ from Proposition~\ref{prop:reconstruction} and then apply the standard $O(n^2)$ matrix-vector multiplication algorithm.
Efficient implementations of near-linear time algorithms for LDR-TD are an interesting area of future work.
\begin{theorem}
  \label{thm:algo}
  Define the simultaneous computation of $k$ Fast Fourier Transforms (FFT), each with size $m$, to be a \emph{batched FFT} with total size $km$.

  Consider any subdiagonal matrix $\vA \in \R^{n \times n}$ and vectors $\vg, \vh \in \R^n$.
  Then $\cK(\vA,\vg)^T$ or $\cK(\vA,\vg)$ can be multiplied by any vector $\vx$ by computing $8\log_2(n)$ batched FFTs, each of total size $2n$.
  The total number of computations is $O(n \log^2 n)$.
\end{theorem}

These algorithms are also automatically differentiable, which we use to compute the gradients when learning.
More complete descriptions of these algorithms are presented in Appendix~\ref{sec:DR-proofs}.

\section{Theoretical properties of structured matrices}
\label{sec:theory}

\paragraph{Complexity of LDR neural networks}

The matrices we use~\eqref{eq:reconstruction} are unusual in that the parameters interact multiplicatively (namely in $\vA^i,\vB^i$) to implicitly define the actual layer.
In contrast, fully-connected layers are linear and other structured layers, such as Fastfood and ACDC~\cite{pmlr-v28-le13,yang2015deep,moczulski2015acdc}, are constant degree in their parameters.
However, we can prove that this does not significantly change the learnability of our classes:
\begin{theorem}
  \label{thm:vc}
  Let $\cF$ denote the class of neural networks with $L$ LDR layers, $W$ total parameters, and piecewise linear activations.
  Let $\sign \cF$ denote the corresponding classification functions, i.e. $\{x \mapsto \sign f(x) : f \in \cF\}$.
  The VC dimension of this class is
  \begin{equation*}
    \VCdim(\sign \mathcal{F}) = O(L W \log W).
  \end{equation*}
\end{theorem}
Theorem~\ref{thm:vc} matches the standard bound for unconstrained weight matrices~\cite{bartlett1999almost,bartlett2017nearly}.
This immediately implies a standard PAC-learnable guarantee~\cite{vapnik1998statistical}.
Theorem~\ref{thm:vc} holds for even more general activations and matrices that for example include the broad classes of~\cite{desa2018two}.
The proof is in Appendix~\ref{sec:vc_dim},
and we empirically validate the generalization and sample complexity properties of our class in Section~\ref{sec:generalization}.

\paragraph{Displacement rank and equivariance}
\label{sec:equivariance}
We observe that displacement rank is related to a line of work outside the resource-constrained learning community, specifically on building \textbf{equivariant} (also called covariant in some contexts~\cite{bronstein2017geometric,marcos2017rotation}) feature representations that transform in predictable ways when the input is transformed. 
An equivariant feature map $\Phi$ satisfies
\begin{equation}
  \label{eq:equivariance-simple}
  \Phi(B(x)) = A(\Phi(x))
\end{equation}
for transformations $A,B$ (invariance is the special case when $A$ is the identity)~\cite{lenc2015understanding,dieleman2016exploiting,schmidt2012learning}.
This means that perturbing the input by a transformation $B$ before passing through the map $\Phi$ is equivalent to first finding the features $\Phi$ then transforming by $A$.

Intuitively, LDR matrices are a suitable choice for modeling \emph{approximately equivariant} linear maps, since the residual $\vA\mathbf{\Phi}-\mathbf{\Phi}\vB$ of~\eqref{eq:equivariance-simple} has low complexity.
Furthermore, approximately equivariant maps should retain the compositional properties of equivariance, which LDR satisfies via Proposition~\ref{prop:closure}. 
For example, Proposition~\ref{prop:closure}\ref{prop:closure:product} formalizes the notion that the composition of two approximately equivariant maps is still approximately equivariant. Using this intuition, the displacement representation~\eqref{eq:DR} of a matrix decomposes into two parts: the operators \({\mathbf{A},\mathbf{B}}\) define transformations to which the model is approximately equivariant, and the low complexity residual \(\mathbf{R}\) controls standard model capacity.

Equivariance has been used in several ways in the context of machine learning.
One formulation, used for example to model ego-motions, supposes that~\eqref{eq:equivariance-simple} holds only approximately, and uses a fixed transformation $B$ along with data for~\eqref{eq:equivariance-simple} to learn an appropriate $A$~\cite{agrawal2015learning,lenc2015understanding}.
Another line of work uses the representation theory formalization of equivariant maps~\cite{cohen2016group,DBLP:conf/icml/KondorT18}.
We describe this formulation in more detail and show how LDR satisfies this definition as well in Appendix~\ref{sec:group_rep}, Proposition~\ref{prop:equivariance}.
In contrast to previous settings, which fix one or both of $A,B$, our formulation stipulates that $\Phi$ can be uniquely determined from $A$, $B$, and learns the latter as part of an end-to-end model.
In Section~\ref{sec:visualization} we include a visual example of latent structure that our displacement operators learn, where they recover centering information about objects from a 2D image dataset.

\section{Empirical evaluation}
\label{sec:eval}

\paragraph{Overview}
In Section~\ref{sec:FC} we consider a standard setting of compressing
  a single hidden layer (SHL) neural network and the fully-connected (FC) layer of a
  CNN for image classification tasks. Following previous
  work~\cite{chen2015compressing,sindhwani2015structured}, we test on
  two challenging MNIST variants~\cite{larochelle2007empirical}, and include two
  additional datasets with more realistic objects (CIFAR-10~\cite{krizhevsky2009learning} and NORB~\cite{lecun2004learning}).
  Since SHL models take a single channel as input, we converted CIFAR-10 to grayscale for this task.
  Our classes and the structured baselines are tested across different parameter budgets in order to show tradeoffs between compression and accuracy. 
    As shown in Table~\ref{table:images}, in the SHL model, our methods consistently have higher test accuracy than baselines for compressed training and inference, by 3.14, 2.70, 3.55, and
 3.37 accuracy points on MNIST-bg-rot, MNIST-noise, CIFAR-10, and NORB respectively. In the CNN model, as shown in Table~\ref{table:images} in Appendix~\ref{sec:additional-results}, we found improvements of 5.56, 0.95, and 1.98 accuracy points over baselines on MNIST-bg-rot, MNIST-noise, and NORB respectively.
Additionally, to explore whether learning the displacement operators can facilitate adaptation to other domains, we replace the input-hidden weights in an LSTM for a language modeling task,
  and show improvements of 0.81-30.47 perplexity points compared to baselines at several parameter budgets.
  
  In addition to experiments on replacing fully-connected layers, in Section~\ref{sec:conv} we also replace the convolutional layer of a
  simple CNN while preserving performance within 1.05 accuracy points on CIFAR-10.
  In Section~\ref{sec:generalization}, we consider the effect of a higher parameter budget. By increasing the rank to just $16$, the \LDRSD{} class meets or exceeds the accuracy of the unstructured FC layer in all datasets we tested on, for both SHL and CNN.\footnote{In addition to the results reported in Table~\ref{table:images}, Figure~\ref{fig:rank-vs-accuracy} and Table~\ref{table:images-extended-cnn} in Appendix~\ref{sec:additional-results}, we also found that at rank 16 the \LDRSD{} class on the CNN architecture achieved test accuracies of 68.48\% and 75.45\% on CIFAR-10 and NORB respectively.} 
Appendix~\ref{sec:exp-details} includes more experimental details and protocols.
Our PyTorch code is publicly available at \url{github.com/HazyResearch/structured-nets}.

\subsection{Compressing fully-connected layers}
\label{sec:FC}

\paragraph{Image classification}

\citet{sindhwani2015structured} showed that for a fixed parameter budget, the Toeplitz-like class significantly outperforms several other compression approaches, including Random Edge Removal~\cite{cirecsan2011high}, Low Rank Decomposition~\cite{denil2013predicting}, Dark Knowledge~\cite{hinton2015distilling}, HashedNets~\cite{chen2015compressing}, and HashedNets with Dark Knowledge.
Following previous experimental settings~\cite{chen2015compressing,sindhwani2015structured}, Table~\ref{table:images} compares our proposed classes to several baselines using dense structured matrices to compress the hidden layer of a single hidden layer neural network. In addition to Toeplitz-like, we implement and compare to other classic LDR types, Hankel-like and Vandermonde-like, which were previously indicated as an unexplored possibility~\citep{sindhwani2015structured,zhao2017theoretical}. 
We also show results when compressing the FC layer of a 7-layer CNN based on LeNet in Appendix~\ref{sec:additional-results}, Table~\ref{table:images-extended-cnn}. 
In Appendix~\ref{sec:additional-results}, we show comparisons to additional baselines at multiple budgets, including network pruning~\cite{han2015learning} and a baseline used in~\cite{chen2015compressing}, in which the number of hidden units is adjusted to meet the parameter budget.

\begin{table}[ht!]
  \centering
  \caption{Test accuracy when replacing the hidden layer with structured classes. Where applicable, rank ($r$) is in parentheses, and the number of parameters in the architecture is in italics below each method. Comparisons to previously unexplored classic LDR types as well as additional structured baselines are included, with the ranks adjusted to match the parameter count of LDR-TD where possible. The Fastfood~\cite{yang2015deep} and Circulant~\cite{cheng2015exploration} methods do not have rank parameters, and the parameter count for these methods cannot be exactly controlled. Additional results when replacing the FC layer of a CNN are in Appendix~\ref{sec:additional-results}. Details for all experiments are in Appendix~\ref{sec:exp-details}.}
  \begin{tabular}{@{}lllll@{}}
    \toprule
    \textbf{Method} &  \textbf{MNIST-bg-rot} & \textbf{MNIST-noise} & \textbf{CIFAR-10} & \textbf{NORB}\\ \midrule
       Unstructured     & 44.08 & 65.15 & 46.03 &  59.83                 \\
                 & \textit{\textcolor{gray}{622506}} & \textit{\textcolor{gray}{622506}} & \textit{\textcolor{gray}{1058826}}   & \textit{\textcolor{gray}{1054726}}                       \\ 
        \hhline{=====}
   LDR-TD ($r=1$)  &    \textbf{45.81} & \textbf{78.45} & \textbf{45.33} & \textbf{62.75}                     \\
           & \textit{\textcolor{gray}{14122}} & \textit{\textcolor{gray}{14122}} & \textit{\textcolor{gray}{18442}}   & \textit{\textcolor{gray}{14342}}                      \\ 

         \midrule
     Toeplitz-like ~\cite{sindhwani2015structured} ($r=4$)         & 42.67 & 75.75 & 41.78 &  59.38                    \\ 
                 & \textit{\textcolor{gray}{14122}} & \textit{\textcolor{gray}{14122}} & \textit{\textcolor{gray}{18442}}   & \textit{\textcolor{gray}{14342}}                      \\ 
          \midrule
       Hankel-like ($r=4$)     &     42.23 & 73.65 & 41.40 & 60.09                \\
                  & \textit{\textcolor{gray}{14122}} & \textit{\textcolor{gray}{14122}} & \textit{\textcolor{gray}{18442}}   & \textit{\textcolor{gray}{14342}}                      \\ 
           \midrule
   Vandermonde-like ($r=4$)      & 37.14 & 59.80 & 33.93 & 48.98                   \\
                  & \textit{\textcolor{gray}{14122}} & \textit{\textcolor{gray}{14122}} & \textit{\textcolor{gray}{18442}}   & \textit{\textcolor{gray}{14342}}                      \\ 
            \midrule
        Low-rank ~\cite{denil2013predicting} ($r=4$)     & 35.67 & 52.25 & 32.28 & 43.66                   \\
                   & \textit{\textcolor{gray}{14122}} & \textit{\textcolor{gray}{14122}} & \textit{\textcolor{gray}{18442}}   & \textit{\textcolor{gray}{14342}}                      \\ 
 	\midrule
      Fastfood  ~\cite{yang2015deep}   & 38.13 & 63.55 & 39.64 & 59.02                   \\
                   & \textit{\textcolor{gray}{10202}} & \textit{\textcolor{gray}{10202}} & \textit{\textcolor{gray}{13322}}   & \textit{\textcolor{gray}{9222}}                      \\
          \midrule
    Circulant  ~\cite{cheng2015exploration}  & 34.46 & 65.35 & 34.28 & 46.45                      \\
                   & \textit{\textcolor{gray}{8634}} & \textit{\textcolor{gray}{8634}} & \textit{\textcolor{gray}{11274}}   & \textit{\textcolor{gray}{7174}}                      \\
    \bottomrule
  \end{tabular}
  \label{table:images}
\end{table}

At rank one (the most compressed setting), our classes with learned operators achieve higher accuracy than the fixed operator classes, and on the MNIST-bg-rot, MNIST-noise, and NORB datasets even improve on FC layers of the same dimensions, by 1.73, 13.30, and 2.92 accuracy points respectively on the SHL task, as shown in Table~\ref{table:images}. On the CNN task, our classes improve upon unstructured fully-connected layers by 0.85 and 2.25 accuracy points on the MNIST-bg-rot and MNIST-noise datasets (shown in Table~\ref{table:images-extended-cnn} in Appendix~\ref{sec:additional-results}). As noted above, at higher ranks our classes meet or improve upon the accuracy of FC layers on all datasets in both the SHL and CNN architectures.

Additionally, in Figure~\ref{fig:rank-vs-accuracy} we evaluate the performance of LDR-SD at higher ranks. Note that the ratio of parameters between LDR-SD and the Toeplitz-like or low-rank is $\frac{r+1}{r}$, which becomes negligible at higher ranks.
Figure~\ref{fig:rank-vs-accuracy} shows that at just rank $16$, the LDR-SD class meets or exceeds the performance of the FC layer on all four datasets, by 5.87, 15.05, 0.74, and 6.86 accuracy points on MNIST-bg-rot, MNIST-noise, CIFAR-10, and NORB respectively, while still maintaining at least 20x fewer parameters.

Of particular note is the poor performance of low-rank matrices.
As mentioned in Section~\ref{sec:DR-background},
every fixed-operator class has the same parameterization (a low-rank matrix).
We hypothesize that the main contribution to their marked performance difference is the effect of the learned displacement operator modeling latent invariances in the data, and that the improvement in the displacement rank classes---from low-rank to Toeplitz-like to our learned operators---comes from more accurate representations of these invariances.
As shown in Figure~\ref{fig:rank-vs-accuracy}, broadening the operator class (from Toeplitz-like at $r=1$ to LDR-SD at $r=1$) is consistently a more effective use of parameters than increasing the displacement rank (from Toeplitz-like at $r=1$ to $r=2$). Note that LDR-SD ($r=1$) and Toeplitz-like ($r=2$) have the same parameter count.

\begin{figure}[!t]
  \centering
    \includegraphics[width=\textwidth]{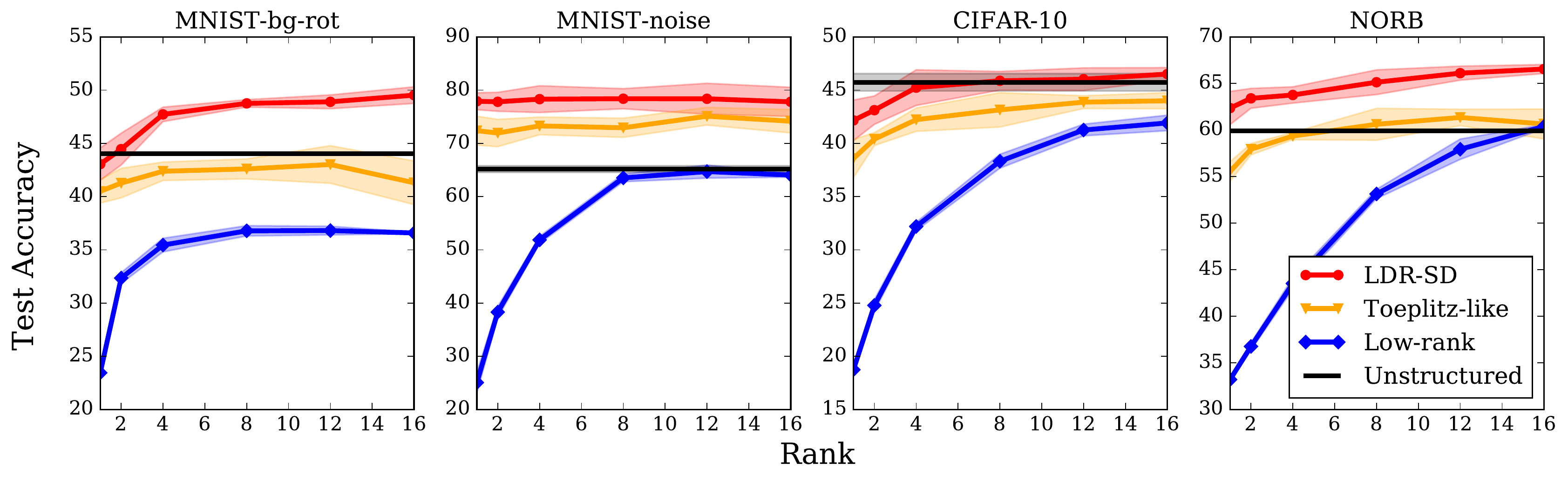}
  \caption{Test accuracy vs. rank for unstructured, \LDRSD{}, Toeplitz-like, low-rank classes. On each dataset, \LDRSD{} meets or exceeds the accuracy of the unstructured FC baseline at higher ranks. At rank 16, the compression ratio of an LDR-SD layer compared to the unstructured layer ranges from $23$ to $30$. Shaded regions represent two standard deviations from the mean, computed over five trials with randomly initialized weights.\vspace*{-0.0em}}
  \label{fig:rank-vs-accuracy}
\end{figure}

For the rest of our experiments outside Section~\ref{sec:FC} we use the algorithms in Appendix~\ref{sec:DR-proofs} specifically for \LDRSD{} matrices, and focus on further evaluation of this class on more expensive models.  

\paragraph{Language modeling}
Here, we replace the input-hidden weights in a single layer long short-term memory network (LSTM) for a language modeling task. We evaluate on the WikiText-2 dataset, consisting of 2M training tokens and a vocabulary size of 33K~\cite{merity2016pointer}.
We compare to Toeplitz-like and low-rank baselines, both previously investigated for compressing recurrent nets~\cite{lu2016learning}.
As shown in Table~\ref{table:lstm}, LDR-SD improves upon the baselines for each budget tested. 
Though our class does not outperform the unstructured model, we did find that it achieves a significantly lower perplexity than the fixed Toeplitz-like class (by 19.94-42.92 perplexity points), suggesting that learning the displacement operator can help adapt to different domains.

\begin{table}[ht!]
  \centering
  \caption{Test perplexity when replacing input-hidden matrices of an LSTM with structured classes on WikiText-2. An unconstrained layer, with 65536 parameters, has perplexity 117.74. Parameter budgets correspond to ranks 1,2,4,8,16,24 for LDR-SD. Lower is better.}
  \begin{tabular}{@{}lllll@{}}
    \toprule
    \textbf{Num. Parameters} & \textbf{LDR-SD} & \textbf{Toeplitz-like} & \textbf{Low-rank} \\ \midrule
    2048                     & \textbf{166.97} & 186.91                 & 205.72 \\
    3072                     & \textbf{154.51} & 177.60                 & 179.46 \\ 
    5120                     & \textbf{141.91} & 178.07                 & 172.38 \\
    9216                     & \textbf{143.60} & 186.52                 & 144.41 \\
    17408                    & \textbf{132.43} & 162.58                 & 135.65 \\
    25600                    & \textbf{129.46} & 155.73                 & 133.37 \\
    \bottomrule
  \end{tabular}
  \label{table:lstm}
\end{table}

\subsection{Replacing convolutional layers}
\label{sec:conv}

Convolutional layers of CNNs are a prominent example of equivariant feature maps.\footnote{Convolutions are designed to be shift equivariant, i.e. shifting the input is equivalent to shifting the output.}
It has been noted that convolutions are a subcase of Toeplitz-like matrices with a particular sparsity pattern\footnote{E.g.\ a $3 \times 3$ convolutional filter on an $n \times n$ matrix has a Toeplitz weight matrix supported on diagonals $-1, 0, 1, n-1, n, n+1, 2n-1, \dots$.}~\cite{cheng2015exploration,sindhwani2015structured}.
As channels are simply block matrices\footnote{A layer consisting of $k$ in-channels and $\ell$ out-channels, each of which is connected by a weight matrix of class $\mathcal{C}$, is the same as a $k \times \ell$ block matrix.}, the block closure property implies that multi-channel convolutional filters are simply a Toeplitz-like matrix of higher rank (see Appendix~\ref{sec:DR-proofs}, Corollary~\ref{cor:toeplitz}).
In light of the interpretation of LDR of an approximately equivariant linear map (as discussed in Section~\ref{sec:equivariance}), we investigate whether replacing convolutional layers with more general representations can recover similar performance, without needing the hand-crafted sparsity pattern.

Briefly, we test the simplest multi-channel CNN model on the CIFAR-10 dataset, consisting of one layer of convolutional channels ($3$ in/out channels), followed by a FC layer, followed by the softmax layer.
The final accuracies are listed in Table~\ref{table:conv-filter}.
The most striking result is for the simple architecture consisting of two layers of a single structured matrix.
This comes within 1.05 accuracy points of the highly specialized architecture consisting of convolutional channels + pooling + FC layer, while using fewer layers, hidden units, and parameters.
The full details are in Appendix~\ref{sec:exp-details}.

\begin{table}[ht]
\setlength\tabcolsep{4.0pt}
  \centering
  \caption{Replacing a five-layer CNN consisting of convolutional channels, max pooling, and FC layers with two generic LDR matrices results in only slight test accuracy decrease while containing fewer layers, hidden units, and parameters. Rank ($r$) is in parentheses.}
  \begin{tabular}{@{}lllll@{}}
    \toprule
    \textbf{First hidden layer(s)}  & \textbf{Last hidden layer} & \textbf{Hidden units} & \textbf{Parameters} & \textbf{Test Acc.} \\ \midrule
    3 Convolutional Channels (CC)   & FC                         & 3072, 512             & 1573089             & 54.59              \\
    3CC + Max Pool                  & FC                         & 3072, 768, 512        & 393441              & 55.14              \\
    4CC + Max Pool                  & FC                         & 4096, 1024, 512       & 524588              & \textbf{60.05}     \\ \midrule
    Toeplitz-like $(r=16)$ channels & Toeplitz-like $(r=16)$     & 3072, 512             & 393216              & 57.29              \\
    \LDRSD{} $(r=16)$ channels      & \LDRSD{} $(r=16)$          & 3072, 512             & 417792              & 59.36              \\
    Toeplitz-like $(r=48)$ matrix   & Toeplitz-like $(r=16)$     & 3072, 512             & 393216              & 55.29              \\
    \LDRSD{} $(r=48)$ matrix        & \LDRSD{} $(r=16)$          & 3072, 512             & 405504              & \textbf{59.00}     \\ \bottomrule
  \end{tabular}
  \label{table:conv-filter}
\end{table}

\subsection{Generalization and sample complexity}
\label{sec:generalization}

Theorem~\ref{thm:vc} states that the theoretical sample complexity of neural
networks with structured weight matrices scales almost linearly in the total
number of parameters, matching the results for networks with fully-connected layers~\cite{bartlett1999almost,bartlett2017nearly}.
As LDR matrices have far fewer parameters, the VC dimension bound for LDR
networks are correspondingly lower than that of general unstructured networks. 
Though the VC dimension bounds are sufficient but not necessary for learnability,
one might still expect to be able to learn over compressed networks with fewer
samples than over unstructured networks.
We empirically investigate this result using the same experimental setting as Table~\ref{table:images} and Figure~\ref{fig:rank-vs-accuracy}. As shown in Table~\ref{table:gen-error} (Appendix~\ref{sec:additional-results}), the structured classes consistently have lower generalization error (measured by the difference between training and test error) than the unstructured baseline.

\paragraph{Reducing sample complexity}
We investigate whether LDR models with learned displacement operators require fewer samples to achieve the same test error, compared to unstructured weights, in both the single hidden layer and CNN architectures. Tables~\ref{table:sample-complexity-shl} and~\ref{table:sample-complexity-cnn} in Appendix~\ref{sec:additional-results} show our results.
In the single hidden layer architecture, when using only 25\% of the training data the LDR-TD class exceeds the performance of an unstructured model trained on the full MNIST-noise dataset. On the CNN model, only 50\% of the training data is sufficient for the LDR-TD to exceed the performance of an unstructured layer trained on the full dataset.

\subsection{Visualizing learned weights}
\label{sec:visualization}
Finally, we examine the actual structures that our models learn.
Figure~\ref{fig:visualization}(a,b) 
shows the heat map of the weight matrix $\mathbf{W} \in \R^{ 784 \times 784 }$ for the Toeplitz-like and \LDRSD{} classes, trained on MNIST-bg-rot with a single hidden layer model.
As is convention, the input is flattened to a vector in $\R^{784}$. 
The Toeplitz-like class is unable to determine that the input is actually a $28 \times 28$ image instead of a vector.
In contrast, LDR-SD class is able to pick up regularity in the
input, as the weight matrix displays grid-like periodicity of size 28.

Figure~\ref{fig:visualization}(c) reveals why the weight matrix displays this pattern.
The equivariance interpretation (Section~\ref{sec:equivariance}) predicts that $\vB$ should encode a meaningful transformation of the inputs.
The entries of the learned subdiagonal are in fact recovering a latent invariant of the 2D domain: when visualized as an image, the pixel intensities correspond to how the inputs are centered in the dataset (Figure~\ref{fig:visualization}(d)).
Figure~\ref{fig:visualization-NORB} in Appendix~\ref{sec:additional-results} shows a similar figure for the NORB dataset, which has smaller objects, and we found that the subdiagonal learns a correspondingly smaller circle.

\begin{figure}[!t]
  \centering
  \begin{subfigure}{0.25\linewidth}
    \centering
    \includegraphics[width=\linewidth]{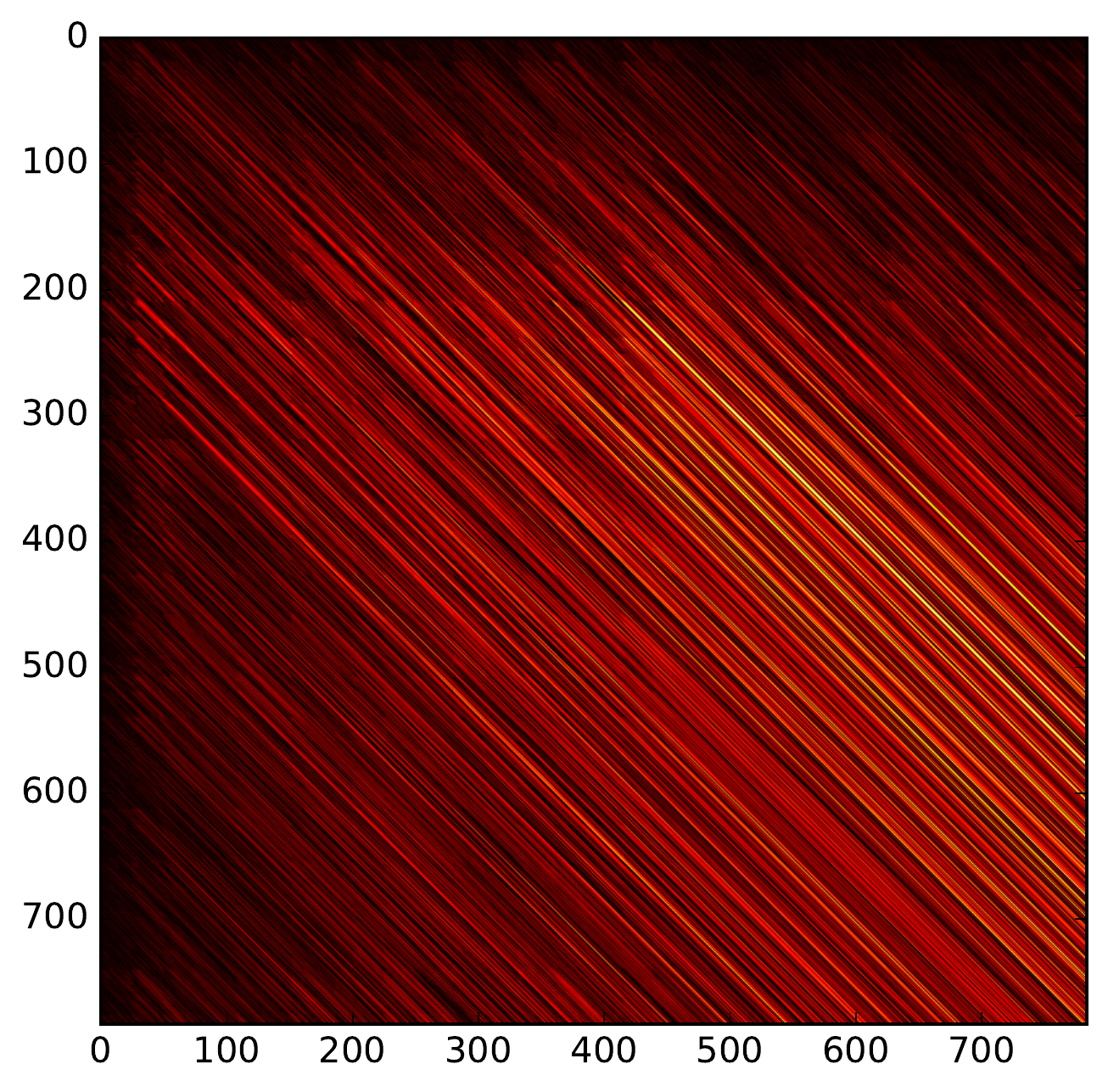}
    \caption{Toeplitz-like}
    \label{fig:heatmap_weight_toeplitz}
  \end{subfigure}\hfill
  \begin{subfigure}{0.25\linewidth}
    \centering
    \includegraphics[width=\linewidth]{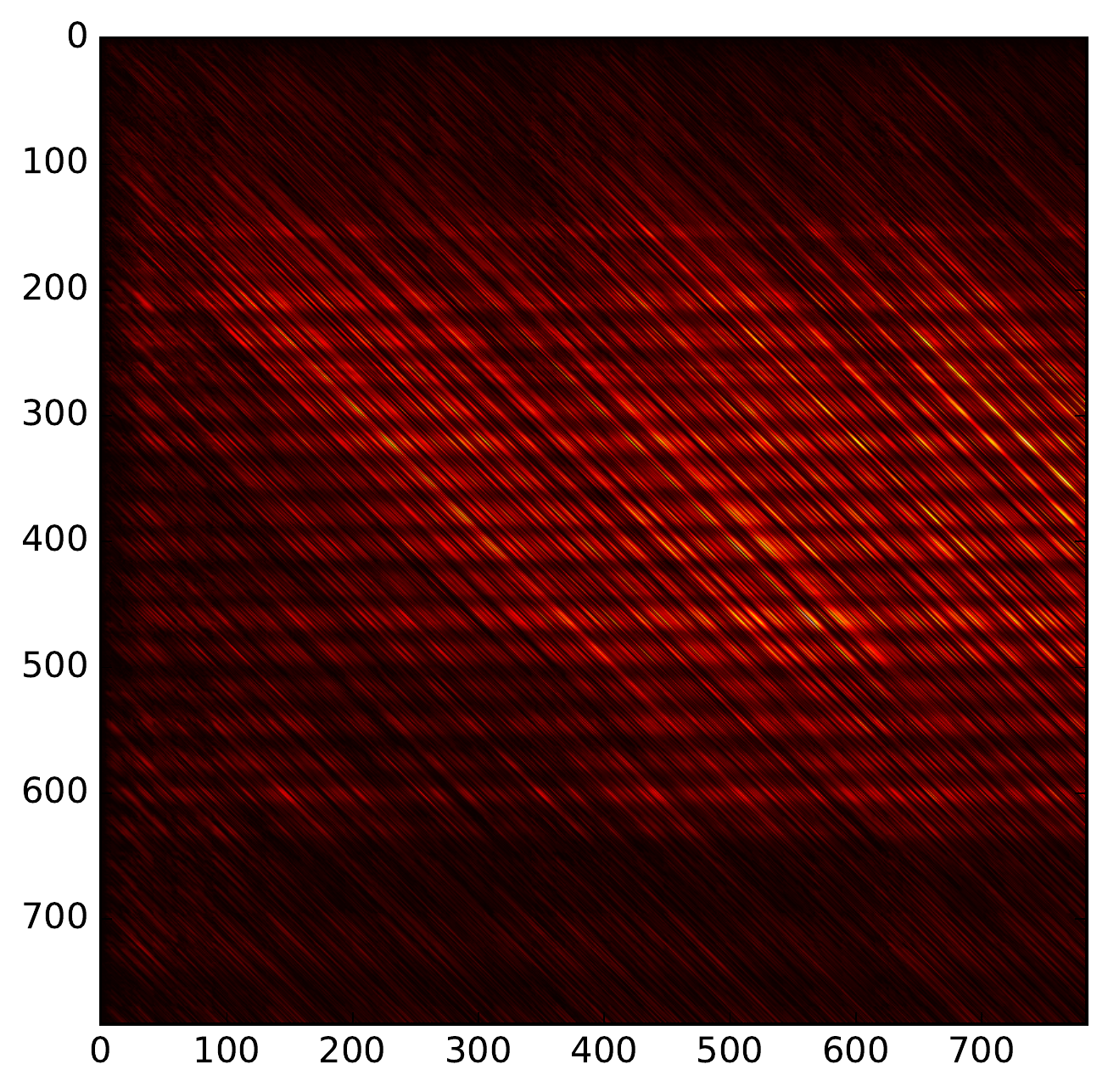}
    \caption{LDR-SD}
    \label{fig:heatmap_weight_subdiag}
  \end{subfigure}\hfill
  \begin{subfigure}{0.25\linewidth}
    \centering
    \includegraphics[width=\linewidth]{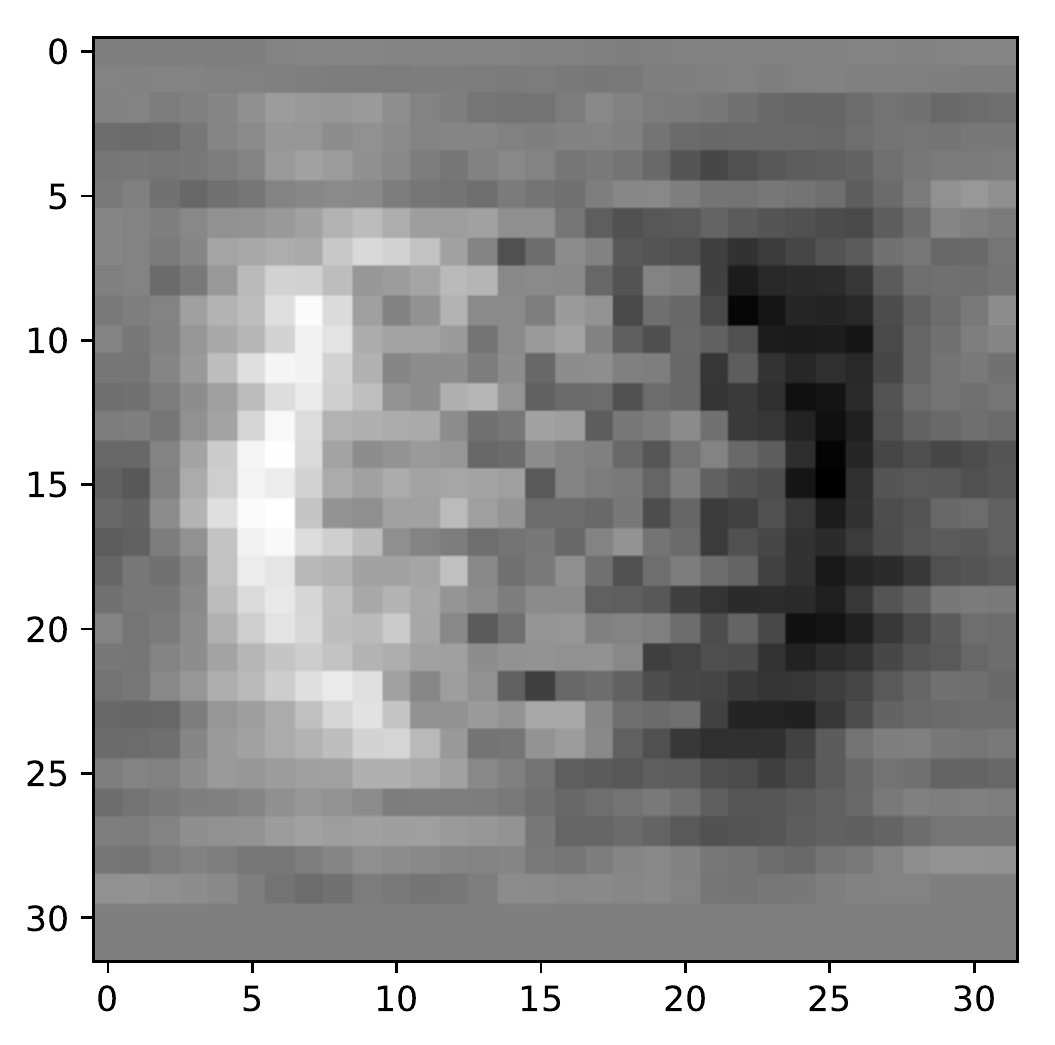}
    \caption{Subdiagonal of $\vB$}
    \label{fig:subdiag_bgrot}
  \end{subfigure}\hfill
  \begin{subfigure}{0.25\linewidth}
    \centering
    \includegraphics[width=\linewidth]{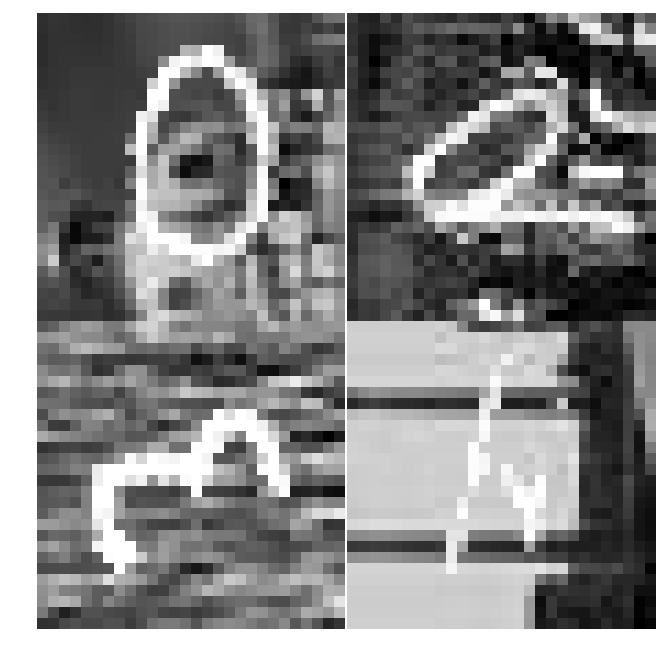}
    \caption{Input examples}
    \label{fig:digits_bgrot}
  \end{subfigure}\hfill

  \caption{The learned weight matrices (a,b) of models trained on MNIST-bg-rot. Unlike the Toeplitz-like matrix, the LDR-SD matrix displays grid-like periodicity corresponding to the 2D input. Figure (c) shows the values of the subdiagonal of $\vB$, reshaped as an image. The size and location of the circle roughly corresponds to the location of objects of interest in the 2D inputs. A similar centering phenomenon was found on the NORB dataset, shown in Figure~\ref{fig:visualization-NORB} in Appendix~\ref{sec:additional-results}.}
  \label{fig:visualization}
\end{figure}

\section{Conclusion}
We generalize the class of low displacement rank matrices explored in machine learning by considering classes of LDR matrices with displacement operators that can be learned from data.
We show these matrices can improve performance on downstream tasks compared to compression baselines and, on some tasks, general unstructured weight layers.
We hope this work inspires additional ways of using structure to achieve both more compact and higher quality representations, especially for deep learning models, which are commonly acknowledged to be overparameterized.

\subsubsection*{Acknowledgments}
We thank Taco Cohen, Jared Dunnmon, Braden Hancock, Tatsunori Hashimoto, Fred Sala, Virginia Smith, James Thomas, Mary Wootters, Paroma Varma, and Jian Zhang for helpful discussions and feedback. 

We gratefully acknowledge the support of DARPA under Nos.\ FA87501720095 (D3M) and FA86501827865 (SDH), NIH under No.\ N000141712266 (Mobilize), NSF under Nos.\ CCF1763315 (Beyond Sparsity) and CCF1563078 (Volume to Velocity), ONR under No.\ N000141712266 (Unifying Weak Supervision), the Moore Foundation, NXP, Xilinx, LETI-CEA, Intel, Google, NEC, Toshiba, TSMC, ARM, Hitachi, BASF, Accenture, Ericsson, Qualcomm, Analog Devices, the Okawa Foundation, and American Family Insurance, and members of the Stanford DAWN project: Intel, Microsoft, Teradata, Facebook, Google, Ant Financial, NEC, SAP, and VMWare. The U.S.\ Government is authorized to reproduce and distribute reprints for Governmental purposes notwithstanding any copyright notation thereon. Any opinions, findings, and conclusions or recommendations expressed in this material are those of the authors and do not necessarily reflect the views, policies, or endorsements, either expressed or implied, of DARPA, NIH, ONR, or the U.S.\ Government.

\bibliography{nips_2018}
\bibliographystyle{plainnat}

\newpage
\appendix

\section{Symbols and abbreviations}
\begin{table}[ht!]
\begin{center}
\caption{Symbols and abbreviations used in this paper.}
\begin{tabular}{ |c|l| } 
\hline
Symbol & Used For \\ 
 \hline
 LDR & low displacement rank \\ 
  \LDRSD{} & matrices with low displacement rank with respect to subdiagonal operators \\
 \LDRTD{} & matrices with low displacement rank with respect to tridiagonal operators\\
 $(\vA,\vB)$ & displacement operators \\ 
 \(\nabla_{\mathbf{A},\mathbf{B}}[\mathbf{M}]\) & Sylvester displacement, $\vA\vM - \vM\vB$ \\ 
 $r$ &  (displacement) rank \\ 
$(\vG,\vH)$ & parameters which define the rank $r$ residual matrix $\vG\vH^T$, where $\vG,\vH \in \mathbb{R}^{n \times r}$ \\ 
$\mathbf{Z_f}$ & unit-f-circulant matrix, defined as $\mathbf{Z_f} = \begin{bmatrix}\mathbf{e_2},\mathbf{e_3},...,\mathbf{e_n},f\mathbf{e_1} \end{bmatrix}$ \\ 
$\cK(\vA,\vv)$  & Krylov matrix, with $i^{th}$ column $\vA^i\vv$\\ 
 $\cD_{\vA,\vB}^r$ & matrices of displacement rank $\le r$ with respect to $(\vA,\vB)$ \\
    $\Phi$ & feature map\\
  CC & convolutional channels\\
    FC & fully-connected\\
 \hline
\end{tabular}
\end{center}
\end{table}

\section{Related work}
\label{sec:related-work}

Our study of the potential for structured matrices for compressing deep learning pipelines was motivated by exciting work along these lines from~\citet{sindhwani2015structured}, the first to suggest the use of low displacement rank (LDR) matrices in deep learning.
They specifically explored applications of the Toeplitz-like class, and empirically show that this class is competitive against many other baselines for compressing neural networks on image and speech domains.
Toeplitz-like matrices were similarly found to be effective at compressing RNN and LSTM architectures on a voice search task~\cite{lu2016learning}.
Another special case of LDR matrices are the circulant (or block-circulant) matrices, which have also been used for compressing CNNs~\cite{cheng2015exploration}; 
more recently, these have also been further developed and shown to achieve state-of-the-art results on FPGA and ASIC platforms~\cite{ding2017circnn}. Earlier works on compressing deep learning pipelines investigated the use of low-rank matrices~\cite{sainath2013low,denil2013predicting}---perhaps the most canonical type of dense structured matrix---which are also encompassed by our framework, as shown in Proposition~\ref{prop:richness}.
Outside of deep learning,~\citet{choromanski2016recycling} examined a structured matrix class that includes Toeplitz-like, circulant, and Hankel matrices (which are all LDR matrices) in the context of kernel approximation.

On the theoretical side, Zhao et al.~\citep{zhao2017theoretical} study properties of neural networks with LDR weight matrices, proving results including a universal approximation property and error bounds.
However, they retain the standard paradigm of fixing the displacement operators and varying the low-rank portion.
Another natural theoretical question that arises with these models is whether the resulting hypothesis class is still efficiently learnable, especially when learning the structured class (as opposed to these previous fixed classes).
Recently,~\citet{pmlr-v80-oymak18a} proved a Rademacher complexity bound for one layer neural networks with low-rank weight matrices.
To the best of our knowledge, Theorem~\ref{thm:vc} provides the first sample complexity bounds for neural networks with a broad class of structured weight matrices including low-rank, our LDR classes, and other general structured matrices~\cite{desa2018two}.

In Section~\ref{sec:our-approach} we suggest that the LDR representation enforces a natural notion of approximate equivariance and satisfies closure properties that one would expect of equivariant representations. 
The study of equivariant feature maps is of broad interest for constructing more effective representations when known symmetries exist in underlying data.
Equivariant linear maps have long been used in algebraic signal processing to derive efficient transform algorithms~\cite{egner2001automatic,egner2004symmetry}.
The fact that convolutional networks induce equivariant representations, and the importance of this effect on sample complexity and generalization, has been well-analyzed~\cite{cohen2016group,Anselmi:2016:ULI:2951995.2952046,Giles:87,pmlr-v54-sokolic17a}.
Building upon the observation that convolutional filters are simply linear maps constructed to be translation equivariant\footnote{Shifting the input to a convolutional feature map is the same as shifting the output.},
exciting recent progress has been made on crafting representations invariant to more complex symmetries such as the spherical rotation group~\cite{s.2018spherical} and egomotions~\cite{agrawal2015learning}.
Generally, however, underlying assumptions are made about the domain and invariances present in order to construct feature maps for each application.
A few works have explored the possibility of learning invariances automatically from data, and design deep architectures that are in principle capable of modeling and learning more general symmetries~\cite{gens2014deep,jaderberg2015spatial,pal2018non}.

\section{Properties of displacement rank}
\label{sec:DR-proofs}

Displacement rank has traditionally been used to describe the Toeplitz-like, Hankel-like, Vandermonde-like, and Cauchy-like matrices, which are ubiquitous in disciplines such as engineering, coding theory, and computer algebra.
Their associated displacement representations are shown in Table~\ref{table:displacements}.

\begin{center}
\begin{table}[ht!]
\centering
\caption{Traditional classes of structured matrices analyzed with displacement rank. In the Vandermonde and Cauchy cases, the displacement operators are parameterized by $v \in \mathbb{R}^n$ and $s, t \in \mathbb{R}^n$ respectively.}
    \begin{tabular}{| l | l | l | l |}
    \hline
    Structured Matrix $\mathbf{M}$ & $\mathbf{A}$ & $\mathbf{B}$ & Displacement Rank $r$ \\ \hline
    Toeplitz & $\mathbf{Z}_1$ & $\mathbf{Z}_{-1}$ & $\leq 2$  \\ \hline
    Hankel & $\mathbf{Z}_1$ & $\mathbf{Z}_{0}^T$ & $\leq 2$  \\ \hline
    Vandermonde & $\diag(v)$ & $\mathbf{Z}_0$ & $\leq 1$  \\ \hline
    Cauchy & $\diag(s)$ & $\diag(t)$ & $\leq 1$  \\ \hline
\end{tabular}
\label{table:displacements}
\end{table}
\end{center}

\begin{proof}[Proof of Proposition~\ref{prop:closure}]
  The following identities are easily verified:
  \begin{description}
    \item[Transpose]
    \[
      \nabla_{\vB^T,\vA^T} \vM^T = -\left( \nabla_{\vA,\vB} \vM \right)^T
    \]
    \item[Inverse]
    \[
      \nabla_{\vB,\vA} \vM^{-1} = -\vM^{-1} \left( \nabla_{\vA,\vB} \vM \right) \vM^{-1}
    \]
    \item[Sum]
    \[
      \nabla_{\vA,\vB} (\vM+\vN) = \nabla_{\vA,\vB} \vM + \nabla_{\vA,\vB} \vN
    \]
    \item[Product]
    \[
      \nabla_{\vA,\vC} \vM\vN = (\nabla_{\vA,\vB} \vM)\vN + \vM\left( \nabla_{\vB,\vC}\vN \right)
    \]
    \item[Block]
    The remainder
    \[ \diag(\vA_1, \dots, \vA_k) \vM - \vM \diag(\vB_1, \dots, \vB_\ell) \]
    is the block matrix
    \[
      (\nabla_{\vA_i,\vB_j} \vM_{ij})_{1 \le i \le k, 1 \le j \le \ell}.
    \]
    This is the sum of $k\ell$ matrices of rank $r$ and thus has rank $rk\ell$.
  \end{description}
\end{proof}

\begin{corollary}
\label{cor:toeplitz}
  A $k \times \ell$ block matrix $\vM$, where each block is a Toeplitz-like matrix of displacement rank $r$, is Toeplitz-like with displacement rank $rk\ell + 2k + 2\ell$.
\end{corollary}
\begin{proof}
  Apply Proposition~\ref{prop:closure:block} where each $\vA_k,\vB_k$ has the form $\vZ_f$.
  Let $\vA = \diag(\vA_1, \dots, \vA_k)$ and $\vB = \diag(\vB_1, \dots, \vB_\ell)$.
  Note that $\vA$ and $\vZ_1$ (of the same size as $\vA$) differ only in $2k$ entries, and similarly $\vB$ and $\vZ_{-1}$ differ in $2\ell$ entries.
  Since an $s$-sparse matrix also has rank at most $s$,
  \begin{align*}
    \vZ_1 \vM - \vM \vZ_{-1} &= \vA\vM - \vM\vB + (\vZ_1-\vA)\vM - \vM(\vZ_{-1}-\vB)
  \end{align*}
  has rank at most $rk\ell + 2k + 2\ell$.
\end{proof}

\begin{proof}[Proof of Proposition~\ref{prop:reconstruction}]
  First consider the rank one case, $\vR = \vg\vh^T$.
  It is easy to check that $\nabla_{\vA^{-1},\vZ^T}\cK(\vA,\vg)$ will only be non-empty in the first column, hence $\cK(\vA,\vg) \in \cD_{\vA^{-1},\vZ^T}^1$.
  Similarly, $\cK(\vB^T,\vh) \in \cD_{\vB^T,\vZ}^1$ and Proposition~\ref{prop:closure}\ref{prop:closure:transpose-inverse} implies $\cK(\vB^T,\vh)^T \in \cD_{\vZ^T,\vB}^1$.
  Then Theorem~\ref{prop:closure}\ref{prop:closure:product} implies that $\cK(\vA,\vg)\cK(\vB,\vh)^T \in \cD_{\vA,\vB}^2$.
  The rank $r$ case follows directly from Theorem~\ref{prop:closure}\ref{prop:closure:sum}.
\end{proof}

\subsection{Expressiveness}
\label{sec:expressiveness-proof}

Expanding on the claim in Section~\ref{sec:our-approach},
we formally show that these structured matrices are contained in the tridiagonal (plus corners) LDR class.
This includes several types previously used in similar works.

  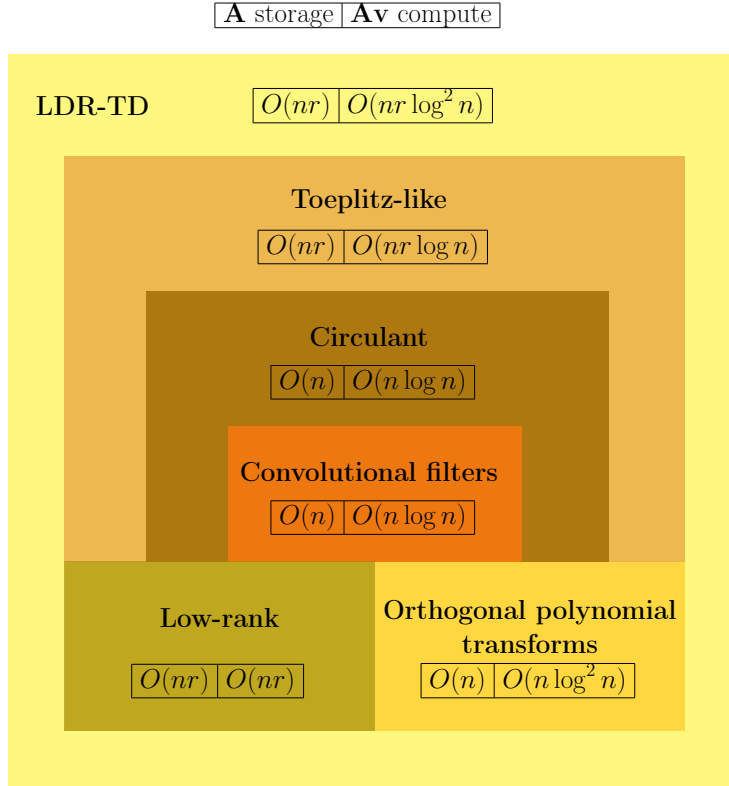
\begin{figure}[htb!]
    \centering
	\begin{tikzpicture}[scale=3]

\node[below right] at (-4.25,0.5) {
{
\begin{tabu}{|c|c|}
\hline
\huge $\vA$ storage & \huge $\vA\textbf{v}$ compute\\
\hline

\end{tabu}
}
};

\path [fill=black, opacity=1.0] (-4.775, -4.5) -- (-0.675, -4.5) -- (-0.675,-2.1) -- (-4.775, -2.1) -- (-4.775, -4.5);
\path [fill=red, opacity=1.0] (-4.05, -4.5) -- (-1.45, -4.5) -- (-1.45,-3.3) -- (-4.05, -3.3) -- (-4.05, -4.5);
\path [fill=orange, opacity=0.5] (-2.75, -6.0) -- (0.0, -6.0) -- (0.0,-4.5) -- (-5.5, -4.5) -- (-5.5, -6.0);

\path [fill=black, opacity=0.5] (-5.5, -6.0) -- (-2.75, -6.0) -- (-2.75,-4.5) -- (-5.5, -4.5) -- (-5.5, -6.0);

\path [fill=purple, opacity=0.5] (-5.5, -4.5) -- (0.0, -4.5) -- (0.0,-0.9) -- (-5.5, -0.9) -- (-5.5, -4.5);

\path [fill=yellow, opacity=0.5] (-6.0, -6.5) -- (0.5, -6.5) -- (0.5,0.0) -- (-6.0, 0.0) -- (-6.0, -6.5);
\node[] at (-5.25, -0.45) {\huge \textbf{LDR-TD}};
\node[] at (-2.8, -1.3) {\huge \textbf{Toeplitz-like}};
\node[] at (-2.8, -2.5) {\huge \textbf{Circulant}};
\node[] at (-2.8, -3.7) {\huge \textbf{Convolutional filters}};
\node[above] at (-4.125, -5.12) {\huge \textbf{Low-rank}};
\node[above] at (-1.375, -5.10) {\huge \textbf{Orthogonal polynomial}};
\node[above] at (-1.375, -5.35) {\huge \textbf{transforms}};

\node[above] at (-4.125, -5.75) {
{
\huge
\begin{tabu}{|c|c|}
\hline
$O(nr)$ & $O(nr)$\\
\hline
\end{tabu}
}
};

\node[above] at (-1.375, -5.75) {
{
\huge
\begin{tabu}{|c|c|}
\hline
$O(n)$ & $O(n \log^2 n)$\\
\hline
\end{tabu}
}
};

\node[above] at (-2.75, -4.3)  {
{
\huge
\begin{tabu}{|c|c|}
\hline
$O(n)$ & $O(n \log n)$\\
\hline
\end{tabu}
}
};

\node[above] at (-2.75, -3.1)  {
{
\huge
\begin{tabu}{|c|c|}
\hline
$O(n)$ & $O(n \log n)$\\
\hline
\end{tabu}
}
};

\node[above] at (-2.75, -1.9)  {
{
\huge
\begin{tabu}{|c|c|}
\hline
$O(nr)$ & $O(nr \log n)$\\
\hline
\end{tabu}
}
};

\node[above] at (-2.75, -0.65)  {
{
\huge
\begin{tabu}{|c|c|}
\hline
$O(nr)$ & $O(nr \log^2 n)$\\
\hline
\end{tabu}
}
};

\end{tikzpicture}
    \caption{Our proposed LDR-TD structured matrix class contains a number of other classes including Toeplitz-like~\cite{sindhwani2015structured} (and other classic displacement types, such as Hankel-like, Vandermonde-like, and Cauchy-like), low-rank~\cite{denil2013predicting}, circulant~\cite{cheng2015exploration}, standard convolutional filters, and orthogonal polynomial transforms, including the Discrete Fourier and Cosine Transforms. Captions for each class show storage cost and operation count for matrix-vector multiplication.}
	\label{fig:expressivity}
  \end{figure}

\paragraph{Classic displacement rank}
The Toeplitz-like, Hankel-like, Vandermonde-like, and Cauchy-like matrices are defined as having LDR with respect to $\vA, \vB \in \{ \vZ_f, \vZ_f^T, \mathbb{D} \}$ where $\mathbb{D}$ is the set of diagonal matrices~\cite{pan2012structured}.
(For example, \citep{sindhwani2015structured} defines the Toeplitz-like matrices as $(\vA,\vB) = (\vZ_1, \vZ_{-1})$.)
All of these operator choices are only non-zero along the three main diagonals or opposite corners, and hence these classic displacement types belong to the LDR-TD class.

\paragraph{Low-rank}
A rank $r$ matrix $R$ trivially has displacement rank $r$ with respect to $(\vA,\vB) = (\vI, \mathbf{0})$.
It also has displacement rank $r$ with respect to $(\vA, \vB) = (\vZ_1, \mathbf{0})$,
since $\vZ_1$ is full rank (it is a permutation matrix) and so $\rank(Z_1 R) = \rank(R) = r$.
Thus low-rank matrices are contained in both the LDR-TD and LDR-SD classes.

\paragraph{Orthogonal polynomial transforms}

The \textbf{polynomial transform} matrix $\vM$ with respect to polynomials $(p_0(X), \dots, p_{m-1}(X))$ and nodes $(\lambda_0, \dots, \lambda_{n-1})$ is defined by $\vM_{ij} = p_i(\lambda_j)$.
When the $p_i(X)$ are a family of orthogonal polynomials, it is called an \textbf{orthogonal polynomial transform}.
\begin{proposition}
  Orthogonal polynomial transforms have displacement rank $1$ with respect to tridiagonal operators. 
\end{proposition}
\begin{proof}
  Every orthogonal polynomial family satisfies a three-term recurrence
  \begin{equation}
    \label{eq:OP}
    p_{i+1}(X) = (a_iX + b_i)p_i(X) + c_i p_{i-1}(X)
  \end{equation}
  where $a_i > 0$~\cite{chihara}.
  Let $\vM$ be an orthogonal polynomial transform with respect to polynomials $(p_i(X))_{0 \le i < m}$ and nodes $(\lambda_j)_{0 \le j < n}$.
  Define the tridiagonal and diagonal matrix
  \begin{align*}
    &\vA =
    \begin{bmatrix}
      -\frac{b_0}{a_0} & \frac{1}{a_0} & 0 & \dots & 0 & 0 \\
      -\frac{c_1}{a_1} & -\frac{b_1}{a_1} & \frac{1}{a_1} & \dots & 0 & 0 \\
      0 & -\frac{c_1}{a_1} & -\frac{b_1}{a_1} & \dots & 0 & 0 \\
      \vdots & \vdots & \vdots & \ddots & \vdots & \vdots \\
      0 & 0 & 0 & \dots & -\frac{b_{m-2}}{a_{m-2}} & \frac{1}{a_{m-2}} \\
      0 & 0 & 0 & \dots & -\frac{c_{m-1}}{a_{m-1}} & -\frac{b_{m-1}}{a_{m-1}} \\
    \end{bmatrix}
    \\&\vB = \diag(\lambda_0, \dots, \lambda_{n-1}).
  \end{align*}
  For any $i \in \{0, \dots, m-2\}$ and any $j$, consider entry $ij$ of $\vA \vM - \vM \vB$.
  This is
  \[
    \frac{1}{a_i}\left[ -c_ip_{i-1}(\lambda_j)  - b_ip_i(\lambda_j) + p_{i+1}(\lambda_j) - \lambda_jp_i(\lambda_j) \right]
  \]
  which is $0$ by plugging $\lambda_j$ into~\eqref{eq:OP}.
  
  Thus $\nabla_{\vA,\vB}\vM$ can only non-zero in the last row, so $\vM \in \cD_{\vA,\vB}^1$.
\end{proof}

\paragraph{Fourier-like transforms}
Orthogonal polynomial transforms include many special cases.
We single out the Discrete Fourier Transform (DFT) and Discrete Cosine Transform (DCT) for their ubiquity.

The $N \times N$ DFT and DCT (type II) are defined as matrix multiplication by the matrices
\begin{align*}
  \vF &= \left(e^{-2\pi\frac{ij}{N}} \right)_{ij} \\
  \vC &= \left(\cos\left[ \frac{\pi}{N}i(j+1/2) \right] \right)_{ij}
\end{align*}
respectively.

The former is a special type of Vandermonde matrix, which were already shown to be in LDR-TD.
Also note that Vandermonde matrices $(\lambda_j^i)_{ij}$ are themselves orthogonal polynomial transforms with $p_i(X) = X^i$.

The latter can be written as
\[
  \left( T_i\left(  \cos\left[ \frac{\pi}{N}(j+\frac{1}{2}) \right] \right) \right)_{ij},
\]
where $T_i$ are the \textbf{Chebyshev polynomials} (of the first kind) defined such that
\[
  T_n(X) = \cos(n \arccos x).
\]
Thus this is an orthogonal polynomial transform with respect to the Chebyshev polynomials.

\paragraph{Other constructions}
From these basic building blocks, interesting constructions belonging to LDR-TD can be found via the closure properties.
For example, several types of structured layers inspired by convolutions, including Toeplitz~\cite{sindhwani2015structured}, circulant~\cite{cheng2015exploration} and block-circulant~\cite{ding2017circnn} matrices,
are special instances of Toeplitz-like matrices.
We also point out a more sophisticated layer~\cite{moczulski2015acdc} in the tridiagonal LDR class, which requires more deliberate use of Proposition~\ref{prop:closure} to show.
\begin{proposition}
  The $\vA\vC\vD\vC^{-1}$ layer, where $\vA,\vD$ are diagonal matrices and $\vC$ is the Discrete Cosine Transform~\citep{moczulski2015acdc}, has displacement rank $2$ with respect to tridiagonal operators.
\end{proposition}
\begin{proof}
  Let $\vT,\Lambda$ be the tridiagonal and diagonal matrix such that $\vC \in \cD_{\vT,\Lambda}^1$.
  Define $\vS = \vA \vT \vA^{-1}$, which is also tridiagonal.
  Note that $\vA \in \cD_{\vS,\vT}^0$ by construction.
  Also note that $\vD \in \cD_{\Lambda,\Lambda}^0$ since $\Lambda$ is diagonal.
  An application of the inverse closure rule yields $\vC \in \cD_{\Lambda,\vT}^1$.
  Finally, the product closure property implies that
  \[
    \vA\vC\vD\vC^{-1} \in \cD_{\vS,\vT}^2.
  \]
\end{proof}

\subsection{Algorithm derivation and details}
\label{sec:algo}

De Sa et al.\ recently showed that a very general class of LDR matrices have asymptotically fast matrix-vector multiplication algorithms~\citep{desa2018two}.
However, parts of the argument are left to existential results.
Building upon~\citet{desa2018two}, we derive a simplified and self-contained algorithm for multiplication by LDR matrices with subdiagonal operators.

Since these matrices can be represented by the Krylov product formula~\eqref{eq:reconstruction}, it suffices to show multiplication algorithms separately for matrix-vector multiplication by $\cK(\vA,\vv)^T$ and $\cK(\vA, \vv)$.

\paragraph{Krylov transpose multiplication}
Let $\vA \in \R^{n \times n}$ be a subdiagonal matrix, i.e. $\vA_{i+1,i}$ are the only possible non-zero entries.
Let $\vu,\vv \in \R^n$, we wish to compute the product $\cK(\vA,\vv)^T\vu$.
For simplicity assume $n$ is a power of $2$.

Following~\citep{desa2018two}, the vector
\[
  \vu^T\cK(\vA,\vv) =
  \begin{bmatrix}
    \vu \vv & \vu \vA \vv & \dots & \vu \vA^{n-1} \vv
  \end{bmatrix}
\]
is the coefficient vector of the polynomial in $X$
\begin{align*}
  &\vu \vv + \vu \vA \vv X + \dots + \vu \vA^{n-1} \vv X^{n-1}
  \\&= \sum_{i=0}^\infty \vu \vA^i X^i \vv
  \\&= \vu (\vI - \vA X)^{-1} \vv,
\end{align*}
where we use the observation that $\vA^n = 0$.

By partitioning $\vA$ into $n/2 \times n/2$ blocks, it has the form
$\begin{bmatrix} \vA_0 & \vzero \\ a \ve_1 \ve_{n/2}^T & \vA_1 \end{bmatrix}$,
where $\vA_0,\vA_1$ are subdiagonal matrices of half the size, $a$ is a scalar, and $\ve_i$ are basis vectors.
Let also $\vu_0,\vu_1 \in \R^{n/2}$, $\vv_0,\vv_1 \in \R^{n/2}$ denote the first and second halves of $\vu,\vv$.

By block matrix inversion for triangular matrices $\begin{bmatrix} \vA & \vzero \\ \vC & \vB \end{bmatrix}^{-1} = \begin{bmatrix} \vA^{-1} & \vzero \\ -\vB^{-1} \vC \vA^{-1} & \vB^{-1} \end{bmatrix}$,
this can be written as
\begin{align*}
  & \vu^T (\vI-\vA X)^{-1} \vv = \begin{bmatrix} \vu_0^T & \vu_1^T \end{bmatrix}
    \begin{bmatrix} (\vI-\vA_0 X)^{-1} & \vzero \\ -(\vI-\vA_1 X)^{-1} (-a \ve_1 \ve_{n/2}^TX) (\vI-\vA_0 X)^{-1} & (\vI-\vA_1 X)^{-1} \end{bmatrix}
    \begin{bmatrix} \vv_0 \\ \vv_1 \end{bmatrix}
  \\&= \vu_0^T (\vI-\vA_0 X)^{-1} \vv_0 + \vu_1^T (\vI-\vA_1 X)^{-1} \vv_1
  + aX \left( \vu_1^T (\vI-\vA_1 X)^{-1} \ve_1 \right) \left( \ve_{n/2}^T (\vI-\vA_0 X)^{-1} \vv_0  \right)
\end{align*}
Therefore $\vu^T (\vI-\vA X)^{-1} \vv$ can be computed from
\begin{align*}
  \vu_0^T (\vI-\vA_0 X)^{-1} \vv_0
  \qquad& \vu_1^T (\vI-\vA_1 X)^{-1} \vv_1 \\
  \vu_1^T (\vI-\vA_1 X)^{-1} \ve_1
  \qquad& \ve_{n/2}^T (\vI-\vA_0 X)^{-1} \vv_0
\end{align*}
with an additional polynomial multiplication and 3 polynomial addition/subtractions.

A modification of this reduction shows that the $2 \times 2$ matrix of polynomials $\begin{bmatrix} \vu & \ve_n \end{bmatrix}^T (\vI-\vA X)^{-1} \begin{bmatrix} \vv & \ve_1 \end{bmatrix}$ can be computed from
\begin{align*}
  &\begin{bmatrix} \vu_0 & \ve_n \end{bmatrix}^T (\vI-\vA_0 X)^{-1} \begin{bmatrix} \vv_0 & \ve_1 \end{bmatrix}
  &\begin{bmatrix} \vu_1 & \ve_n \end{bmatrix}^T (\vI-\vA_1 X)^{-1} \begin{bmatrix} \vv_1 & \ve_1 \end{bmatrix}
\end{align*}
with an additional constant number of polynomial multiplications and additions.

The complete recursive algorithm is provided in Algorithm~\ref{alg:KT-recursive}, where subroutine \texttt{R} computes the above matrix of polynomials.
For convenience, Algorithm~\ref{alg:KT-recursive} uses Python indexing notation.

\begin{algorithm}
  \caption{Krylov Transpose (Recursive)}
  \label{alg:KT-recursive}
  \begin{algorithmic}[1]
    \renewcommand{\algorithmicrequire}{\textbf{Input:}}
    \renewcommand{\algorithmicensure}{\textbf{Output:}}
    \Function{Krylov\_Transpose}{$\vA \in \R^{n \times n}, \vu, \vv \in \R^n$}
      \State $\vs \gets \text{subdiagonal}(\vA)$
      \State \Return{\Call{R}{$\vs, \vu, \vv$}}
    \EndFunction
    \Function{R}{$\vs \in \R^{n-1}, \vu, \vv$} 
      \State $S_0 \gets \Call{R}{\vs[0:n/2-1], \vu[0:n/2], \vv[0:n/2]}$
      \State $S_1 \gets \Call{R}{\vs[n/2:n-1], \vu[n/2:n], \vv[n/2:n]}$
      \State\label{alg:KT-recursive:step:polymult} $L \gets \vs[n/2-1]X \cdot \begin{bmatrix} S_1[0,1]\cdot S_0[1,0] & S_1[0,1] \cdot S_0[1,1] \\ S_1[1,1] \cdot S_0[1,0] & S_1[1,1] \cdot S_0[1,1] \end{bmatrix}$
      \State \Return{$\begin{bmatrix} L[0,0] + S_0[0,0] + S_1[0,0] & L[0,1] + S_0[0,1] \\ L[1,0] + S_1[1,0] & L[1,1] \end{bmatrix}$}
    \EndFunction

  \end{algorithmic}
\end{algorithm}

A polynomial multiplication of degree $m$ in Step~\ref{alg:KT-recursive:step:polymult} can be computed as a convolution of size $2m$. This reduces to two Fast Fourier Transform (FFT) calls, an elementwise multiplication in the frequency domain, and an inverse FFT.
The total number of calls can be further reduced to 4 FFTs and 4 inverse FFTs.

Algorithm~\ref{alg:KT-recursive} defines a recursion tree, and in practice we compute this breadth first bottom-up to avoid recursive overhead.
This also allows the FFT operations to be batched and computed in parallel.
Thus the $d$-th layer of the algorithm (starting from the leaves) performs $\frac{n}{2^d}$ FFT computations of size $2^{d+1}$.

This completes the proof of Theorem~\ref{thm:algo}.

We note several optimizations that are useful for implementation:
\begin{enumerate}
  \item The polynomial $\ve_n^T (\vI - \vA_i X)^{-1} \ve_1$ for $i = 0, 1$ are
  in fact monomials, which can be shown inductively.
  To use the notation of Algorithm~\ref{alg:KT-recursive}, $S_0[1, 1]$,
  $S_1[1, 1]$, and $L[1, 1]$ are monomials.
  Therefore the polynomial multiplication with $S_0[1, 1]$ and $S_1[1, 1]$ can
  be done directly by coefficient-wise multiplication instead of using the FFT.

  \item
  \label{enum:sum_across_subproblems}
  We don't need the polynomials $\vu_0^T (\vI - \vA_0 X)^{-1} \vv_0$ and
  $\vu_1^T (\vI - \vA_1 X)^{-1} \vv_1$ separately, we only need their sum.
  To use the notation of Algorithm~\ref{alg:KT-recursive}, we don't need
  $S_0[0, 0]$ and $S_1[0, 0]$ separately, we only need their sum.
  In fact, by tracing the algorithm from the leaves of the recursion tree to the
  root, we see that across the same depth $d$, only the sum of the terms
  $S_0[0, 0] + S_1[0, 0]$ of the $n/2^d$ subproblems is required, not the
  individual terms.
  Therefore, when computing polynomial multiplication at depth $d$, we can
  perform the FFT of size $2^{d+1}$ and the pointwise multiplication, then sum
  across the $n/2^d$ problems before performing the inverse FFT of size
  $2^{d+1}$.
\end{enumerate}

\paragraph{Efficient batching with respect to input vector and rank.}
Optimization~\ref{enum:sum_across_subproblems} is especially important for
efficient multiplication with respect to batched input $\vu$ and higher rank
$\vv$.
Suppose that $\vu$ has size $n \times b$ and there are $r$ vectors
$\vv_1, \dots, \vv_r$, and we wish to compute
$\sum_{i=1}^r \mathcal{K}(\vA, \vv_i)^T \vu$.
Naively performing Algorithm~\ref{alg:KT-recursive} on each of the $b$ inputs
and each of the $r$ vectors then summing the results, takes $O(br n \log^2 n)$ time.
The bottleneck of the algorithm is the polynomial multiplication
$S_1[0, 1] \cdot S_0[1, 0]$.
At depth $d$, there are $n/2^d$ subproblems, and in each of those, $S_1[0, 1]$
consists of $b$ polynomials of degree at most $2^d$, while $S_0[1, 0]$ consists
of $r$ polynomials of degree at most $2^d$.
If we apply optimization~\ref{enum:sum_across_subproblems}, we first perform
the FFT of size $2^{d+1}$ on these $(b + r) n/2^d$ polynomials, then pointwise
multiplication in the frequency domain to get $br n/2^d$ vectors of size
$2^{d+1}$ each.
Next we sum across the $n/2^d$ problems to get $br$ vectors, before
performing the inverse FFT of size $2^{d+1}$ to these $br$ vectors.
The summing step allows us to reduce the number of inverse FFTs from $br n/2^d$
to $br$.
The total running time over all depth $d$ is then
$O((b + r) n \log^2 n + br n \log n)$ instead of $O(br n \log^2 n)$.

\paragraph{Krylov multiplication}

\citet{desa2018two} do not provide explicit algorithms for the more complicated problem of multiplication by $\cK(\vA,\vv)$, instead justifying the existence of such an algorithm with the \textbf{transposition principle}.
Traditional proofs of the transposition principle use circuit based arguments involving reversing arrows in the arithmetic circuit defining the algorithm's computation graph~\cite{transposition}.

Here we show an alternative simple way to implement the transpose algorithm using any automatic differentiation (AD) implementation, which all modern deep learning frameworks include.
AD states that for any computation, its derivative can be computed with only a constant factor more operations~\cite{griewank2008evaluating}.
\begin{proposition}[Transposition Principle]
  \label{prop:transposition}
  If the matrix $\vM \in \R^{n \times n}$ admits matrix-vector multiplication by any vector in $N$ operations, then $\vM^T$ admits matrix-vector multiplication in $O(N+n)$ operations.
\end{proposition}
\begin{proof}
  Note that for any $\vx$ and $\vy$, the scalar $\vy^T \vM \vx = \vy \cdot (\vM \vx)$ can be computed in $N + n$ operations.

  The statement follows from applying reverse-mode AD to compute $\vM^T \vy = \frac{\partial}{\partial \vx}(\vy^T \vM \vx)$.

  Additionally, the algorithm can be optimized by choosing $\vx = \vzero$ to construct the forward graph.
\end{proof}

To perform the optimization mentioned in Proposition~\ref{prop:transposition}, and avoid needing second-order derivatives when computing backprop for gradient descent, we provide an explicit \href{https://github.com/HazyResearch/structured-nets}{implementation} of non-transpose Krylov multiplication $\cK(\vA,\vv)$. This was found by using Proposition~\ref{prop:transposition} to hand-differentiate Algorithm~\ref{alg:KT-recursive}.

Finally, we comment on multiplication by the \LDRTD{} class.
Desa et al.\citep{desa2018two} showed that these matrices also have asymptotically efficient multiplication algorithms, of the order $O(r n\log^3 n)$ operations.
However, these algorithms are even more complicated and involve operations such as inverting matrices of polynomials in a modulus.
Practical algorithms for this class similar to the one we provide for \LDRSD{} matrices require more work to derive.

\subsection{Displacement rank and equivariance}
\label{sec:group_rep}
Here we discuss in more detail the connection between LDR and equivariance. One line of work~\cite{cohen2016group,DBLP:conf/icml/KondorT18} has used the group representation theory formalization of equivariant maps, in which the model is equivariant to a set of transformations which form a group $G$. Each transformation $g \in G$ acts on an input $x$ via a corresponding linear map $T_g$. For example, elements of the rotation group in two and three dimensions, $SO(2)$ and $SO(3)$, can be represented by 2D and 3D rotation matrices respectively.
Formally, a feature map $\Phi$ is equivariant if it satisfies
\begin{equation}
  \label{eq:equivariance}
  \Phi(T_g x) = T_g'(\Phi(x))
\end{equation}
for representations $T, T'$ of $G$~\cite{cohen2016group,DBLP:conf/icml/KondorT18}. This means that perturbing the input $x$ by a transformation $g \in G$  before computing the map $\Phi$ is equivalent to first finding the features $\Phi$ and then applying the transformation.
Group equivariant convolutional neural networks (G-CNNs) are a particular realization where $\Phi$ has a specific form $G\to\R^d$, and $T, T'$ are chosen in advance~\cite{cohen2016group}.
We use the notation $\mathbf{\Phi}$ to distinguish our setting, where the input $x$ is finite dimensional and $\Phi$ is linear.

\begin{proposition}
\label{prop:equivariance}
If $\mathbf{\Phi}$ has displacement rank 0 with respect to invertible $\mathbf{A},\mathbf{B}$, then $\mathbf{\Phi}$ is equivariant as defined by (\ref{eq:equivariance}).
\end{proposition}
\begin{proof}
  Note that if $\mathbf{A}\mathbf{\Phi} = \mathbf{\Phi} \mathbf{B}$ for invertible matrices $\mathbf{A},\mathbf{B}$ (i.e. if a matrix $\mathbf{\Phi}$ has displacement rank 0 with respect to $\mathbf{A}$ and $\mathbf{B}$), then $\mathbf{A}^i\mathbf{\Phi} = \mathbf{\Phi} \mathbf{B}^i$ also holds for $i \in \mathbb{Z}$. Also note that the set of powers of any invertible matrix forms a cyclic group, where the group operation is multiplication. The statement follows directly from this fact, where the group $G$ is $\mathbb{Z}$, and the representations $T$ and $T'$ of $G$ correspond to the cyclic groups generated by $\mathbf{A}$ and $\mathbf{B}$, respectively consisting of $\vA^i$ and $\vB^i$ for all $i \in \mathbb{Z}$.
\end{proof}

More generally, a feature map $\Phi$ satisfying~\eqref{eq:equivariance} for a set of generators $S = \{g_i\}$ is equivariant with respect to the free group generated by $S$. Proposition~\ref{prop:equivariance} follows from the specific case of a single generator, i.e. $S = \{1\}$. 

\section{Bound on VC dimension and sample complexity}
\label{sec:vc_dim}

In this section we upper bound the VC dimension of a neural network where all
the weight matrices are LDR matrices and the activation functions are piecewise
polynomials.
In particular, the VC dimension is almost linear in the number of parameters,
which is much smaller than the VC dimension of a network with unstructured layers.
The bound on the VC dimension allows us to bound the sample complexity to learn
an LDR network that performs well among LDR networks.
This formalizes the intuition that compressed parameterization reduces the
complexity of the class.

\paragraph{Neural network model}
Consider a neural network architecture with $W$ parameters, arranged in $L$
layers.
Each layer $l$, has output dimension $n_{l}$, where $n_0$ is the dimension of
the input data and the output dimension is $n_{L} = 1$.
For $l = 1, \dots, L$, let $\vi_l \in \mathbb{R}^{n_l}$ be the input to the $l$-th
layer.
The input to the $(l+1)$-th layer is exactly the output of the $l$-th layer.
The activation functions $\phi_l$ are piecewise polynomials with at most $p + 1$
pieces and degree at most $k \geq 1$.
The input to the first layer is the data $\vi_1 = \vx \in \mathbb{R}^{n_1}$, and the
output of the last layer is a real number $i_{L+1} \in \mathbb{R}$.
The intermediate layer computation has the form:
\begin{equation*}
  i_{l+1} = \phi_l(\vM_l \vi_l + \vb_l) \quad \text{(applied elementwise)}, \qquad \text{where } \vM_l
  \in \mathbb{R}^{n_{l-1} \times n_{l}},\ \vb_l \in \mathbb{R}^{n_{l}}.
\end{equation*}
We assume the activation function of the final layer is the identity.

Each weight matrix $\vM_l$ is defined through some set of parameters;
for example, traditional unconstrained matrices are parametrized by their entries,
and our formulation~\eqref{eq:reconstruction} is parametrized by the entries of some
operator matrices $\vA_l, \vB_l$ and low-rank matrix $\vG_l\vH_l^T$.
We collectively refer to all the parameters of the neural network (including the biases $b_l$)
as $\theta \in \mathbb{R}^{W}$, where $W$ is the number of parameters.

\paragraph{Bounding the polynomial degree}
The crux of the proof of the VC dimension bound is that the entries of $\vM \in
\mathbb{R}^{n \times m}$
are polynomials in terms of the entries of its parameters ($\vA$, $\vB$, $\vG$, and $\vH$).
of total degree at most $c_1 m^{c_2}$ for universal constants $c_1,c_2$.
This allows us to bound the total degree of all of the layers and apply Warren's
lemma to bound the VC dimension.

We will first show this for the specific class of matrices that we use, where the matrix $\vM$
is defined through equation~\eqref{eq:reconstruction}.

\begin{lemma}
  Suppose that $\vM \in \R^{m \times m}$ is defined as
  \begin{equation*}
    \vM = \sum_{i=1}^r \cK(\vA,\vg_i)\cK(\vB^T,\vh_i).
  \end{equation*}
  Then the entries of $\vM$ are polynomials of the entries of $\vA,\vB,\vG,\vH$ with total degree at most $2m$.
  \label{lem:degree_bound_ldr}
\end{lemma}
\begin{proof}
  Since $\cK(\vA, \vg_i) = \begin{bmatrix} \vg_i & \vA \vg_i & \dots & \vA^{m-1}
    \vg_i \end{bmatrix}$, and each entry of $\vA^k$ is a polynomial of the entries
  of $\vA$ with total degree at most $k$, the entries of $\cK(\vA, \vg_i)$ are
  polynomials of the entries of $A$ and $\vg_i$ with total degree at most $m$.
  Similarly the entries of $\cK(\vB^T, \vh_i)$ are polynomials of the entries of
  $\vB$ and $\vh_i$ with total degree at most $m$.
  Hence the entries of $\cK(\vA, \vg_i) \cK(\vB^T, \vh_i)$ are polynomials of the
  entries of $\vA, \vB, \vG, \vH$ with total degree at most $2m$.
  We then conclude that the entries of $\vM$ are polynomials of the entries of $\vA,
  \vB, \vG, \vH$ with total degree at most $2m$.
\end{proof}

\begin{lemma}
  Suppose that the LDR weight matrices $M_l$ of a neural network have entries
  that are polynomials in their parameters with total degree at most
  $c_1 n_{l-1}^{c_2}$ for some universal constants $c_1, c_2 \geq 0$.
  For a fixed data point $\vx$, at the $l$-th layer of a neural network with LDR
  weight matrices, each entry of $\vM_l \vi_l + \vb_l$ is a piecewise polynomial
  of the network parameters $\theta$, with total degree at most $d_l$, where
  \begin{equation*}
    d_0 = 0, \qquad d_l = k d_{l-1} + c_1 n_{l-1}^{c_2} \quad \text{for }\ l = 1, \dots, L.
  \end{equation*}
  Thus entries of the output $\phi_l(\vM_l \vi_l + \vb_l)$ are piecewise polynomials of
  $\theta$ with total degree at most $k d_l$.
  Moreover,
  \begin{equation}
    d_l \leq c_1 k^{l-1} \sum_{j=0}^{l-1} n_j^{c_2}.
    \label{eqn:bound_d_l}
  \end{equation}
  \label{lem:degree_bound_nn}
\end{lemma}
By Lemma~\ref{lem:degree_bound_ldr}, Lemma~\ref{lem:degree_bound_nn} applies to
the specific class of matrices that we use, for $c_1 = 2$ and $c_2 = 1$.
As we will see, it also applies to very general classes of structured matrices.

\begin{proof}
  We induct on $l$.
  For $l = 1$, since $\vi_1 = \vx$ is fixed, the entries of $\vM_1$ are polynomials of
  $\theta$ of degree at most $c_1 n_0^{c_2}$, and
  so the entries of $\vM_1 \vi_1 + \vb_1$ are polynomials of $\theta$ with total degree at
  most $d_1 = c_1 n_0^{c_2}$.
  As $\phi$ is a piecewise polynomials of degree at most $k$, each entry the output
  $\phi_1(\vM_1 \vi_1 + \vb_1)$ is a piecewise polynomial of $\theta$ with total degree at
  most $2n_0k$.
  The bound~\eqref{eqn:bound_d_l} holds trivially.

  Suppose that the lemma is true for some $l - 1 \geq 1$.
  Since the entries of $\vi_l$ are piecewise polynomials of $\theta$ with total degree
  at most $k d_{l-1}$ and entries of $\vM_{l}$ are polynomials of $\theta$ with total
  degree at most $c_1 n_{l-1}^{c_2}$, the
  entries of $\vM_l \vi_l + \vb_l$ are piecewise polynomials of $\theta$ with total degree
  at most $d_l = k d_{l-1} + c_1 n_{l-1}^{c_2}$.
  Thus $\phi_l(\vM_l \vi_l + \vb_l)$ have entries that are piecewise polynomials of $\theta$
  with total degree at most $k d_l$.

  We can bound
  \begin{equation*}
    d_l = k d_{l-1} + c_1 n_{l-1}^{c_2} \leq k c_1 k^{l-2} \sum_{j=0}^{l-2} n_j^{c_2} + c_1 n_{l-1}^{c_2} \leq
    c_1 k^{l-1} \sum_{j=0}^{l-1} n_j^{c_2},
  \end{equation*}
  where we have used the fact that $k \geq 1$, so $c_1 n_{l-1}^{c_2} \leq c_1 k^{l-1}
  n_{l-1}^{c_2}$.
  This concludes the proof.
\end{proof}

\paragraph{Bounding the VC dimension}
Now we are ready to bound the VC dimension of the neural network.

\begin{theorem}
  For input $x \in \mathcal{X}$ and parameter $\theta \in \mathbb{R}^{W}$, let $f(x, \theta)$
  denote the output of the network.
  Let $\cF$ be the class of functions $\{ x \to f(x, \theta) \colon \theta \in \mathbb{R}^{W}
  \}$.
  Denote $\sign \cF \defeq \{ x \to \sign f(x, \theta) \colon \theta \in \mathbb{R}^W \}$.
  Let $W_l$ be the number of parameters up to layer $l$ (i.e., the total number
  of parameters in layer $1, 2, \dots, l$).
  Define the effective depth as
  \begin{equation*}
    \bar{L} \defeq \frac{1}{W} \sum_{l=1}^{L} W_l,
  \end{equation*}
  and the total number of computation units (including the input dimension) as
  \begin{equation*}
    U \defeq \sum_{l=0}^{L} n_l.
  \end{equation*}
  Then
  \begin{equation*}
    \VCdim(\sign \mathcal{F}) = O(\bar{L} W \log (p U) + \bar{L} L W \log k).
  \end{equation*}
  In particular, if $k = 1$ (corresponding to piecewise linear networks) then
  \begin{equation*}
    \VCdim(\sign \mathcal{F}) = O(\bar{L} W \log (pU)) = O(L W \log W).
  \end{equation*}
  \label{thm:vc_bound}
\end{theorem}

We adapt the proof of the upper bound from \citet{bartlett1999almost,
  bartlett2017nearly}.
The main technical tool is Warren's lemma \cite{warren1968lower}, which bounds
the growth function of a set of polynomials.
We state a slightly improved form here from \citet[Theorem
8.3]{anthony2009neural}.

\begin{lemma}\label{lem:warren}
	Let $p_1, \ldots, p_m$ be polynomials of degree at most $d$ in $n \leq m$ variables.
    Define
    \[
    	K \defeq |\{(\sign(p_1(\vx)), \ldots, \sign(p_m(\vx)) : \vx \in \R^n \}|,
    \]
    i.e., $K$ is the number of possible sign vectors given by the polynomials.
    Then $K \leq 2(2emd/n)^n$.
\end{lemma}

\begin{proof}[Proof of Theorem~\ref{thm:vc_bound}]
  Fixed some large integer $m$ and some inputs $\vx_1, \dots, \vx_m$.
  We want to bound the number of sign patterns that the neural network can
  output for the set of input $\vx_1, \dots, \vx_m$:
  \begin{equation*}
    K \defeq \abs{ \{ (\sign f(\vx_1, \theta), \dots, \sign f(\vx_m, \theta)) \colon \theta \in \mathbb{R}^W \}}.
  \end{equation*}
  We want to partition the parameter space $\mathbb{R}^{W}$ so that for a fixed
  $\vx_j$, the output $f(\vx_j, \theta)$ is a polynomial on each region in the partition.
  Then we can apply Warren's lemma to bound the number of sign patterns.
  Indeed, for any partition $\cS = \{ P_1, \dots, P_N \}$ of the parameter space
  $\mathbb{R}^W$, we have
  \begin{equation}
    K \leq \sum_{j=1}^{N} \abs{ \{ (\sign f(\vx_1, \theta), \dots, \sign f(\vx_m, \theta)) \colon \theta
      \in P_j \}}.
    \label{eqn:sum_partition}
  \end{equation}

  We construct the partitions iteratively layer by layer, through a sequence
  $\cS_0, \cS_1, \dots, \cS_{L-1}$ of successive refinements, satisfying two
  properties:
  \begin{enumerate}
    \item $\abs{\cS_0} = 1$ and for each $1 \leq l \leq L - 1$,
    \begin{equation*}
      \abs{\cS_l} \leq \abs{\cS_{l-1}} 2 \left( \frac{2 e m p n_l d_l}{W_l} \right)^{W_l},
    \end{equation*}
    where $n_l$ is the dimension of the output of the $l$-th layer, $d_l$ is the
    bound on the total degree of $\vM_l \vi_l + \vb_l$ as piecewise polynomials of $\theta$
    as defined in Lemma~\ref{lem:degree_bound_nn}, and $W_l$ is the number of
    parameters up to layer $l$ (i.e., the total number of parameters in layer
    $1, 2, \dots, l$).

    \item For each $l = 0, \dots, L-1$, for each element $S$ of $\cS_{l}$, for
    each fixed data point $\vx_j$ (with $j = 1, \dots, m$), the entries of the
    output $\phi_l(\vM_l \vi_l + \vb_l)$ when restricted to $S$ are polynomials of $\theta$
    with total degree at most $k d_{l-1}$.
  \end{enumerate}

  We can define $\cS_0=\mathbb{R}^{W}$, which satisfies property 2, since at
  layer 1, the entries of $\vi_1 = \vx_j$ (for fixed $\vx_j$) are polynomials of $\theta$
  of degree $d_0 = 0$.

  Suppose that we have constructed $\cS_0,\dots,\cS_{l-1}$, and we want to
  define $\cS_l$.
  For any $h\in[n_l],j\in[m],$ and $S\in \cS_{l-1}$, let $p_{h,\vx_j,S}(\theta) = (\vM_l \vi_l +
  \vb_l)_h|_S$ be the $h$-th entry of $\vM_l \vi_l + \vb_l$ restricted to the region
  $S$.
  By the inductive hypothesis, for each $S \in \cS_{l-1}$, the entries of $i_l$
  when restricted to $S$ are polynomials of $\theta$ of total degree at most $k
  d_{l-1}$.
  Thus by Lemma~\ref{lem:degree_bound_nn}, the entries of $\vM_l \vi_l + \vb_l$ when
  restricted to $S$ are polynomials of $\theta$ with total degree at most $k d_{l-1}
  + c_1 n_{l-1}^{c_2} = d_l$, and depends on at most $W_{l}$ many variables.

  Since the activation function is piecewise polynomial with at most $p$ pieces,
  let $\{t_1,\dots,t_p\}$ be the set of breakpoints.
  For any fixed $S\in\cS_{l-1}$, by Lemma~\ref{lem:warren}, the polynomials
  \begin{equation*}
    \left\{p_{h,\vx_j,S}(\theta)-t_i : h \in [n_l], j\in[m], i \in[p] \right\}
  \end{equation*}
  can have at most
  \begin{equation*}
    \Pi \defeq 2 \left( \frac{2 e (n_l m p) d_l}{W_l} \right)^{W_l}
  \end{equation*}
  distinct sign patterns when $\theta \in \mathbb{R}^W$.
  We can then partition $\mathbb{R}^W$ into this many regions so that within
  each region, all these polynomials have the same signs.
  Intersecting all these regions with $S$ yields a partition of $S$ into at most
  $\Pi$ subregions.
  Applying this for all $S\in\cS_{l-1}$ gives a partition $\cS_{l}$ that satisfies
  the property 1.

  Fix some $S'\in \cS_n$.
  When $\theta$ is restricted to $S'$, by construction, all the polynomials
  \begin{equation*}
    \left\{p_{h,\vx_j,S}(\theta)-t_i : h \in [n_l], j\in[m], i \in[p] \right\}
  \end{equation*}
  have the same sign.
  This means that the entries of $\vM_l \vi_l + \vb_l$ lie between two
  breakpoints of the activation function, and so the entries of the output
  $\phi_l(\vM_l \vi_l + \vb_l)$ are fixed polynomials in $W_l$ variables of degree
  at most $k d_l$.

  By this recursive construction, $\cS_{L-1}$ is a partition of $\mathbb{R}^W$
  such that for $S\in\cS_{L-1}$ the network output for any input $\vx_j$ is a
  fixed polynomial of $\theta \in S$ of degree at most $k d_{L-1} + c_1
  n_{L-1}^{c_2} = d_L$ (recall that we assume the activation function of the
  final layer is the identity).
  Hence we can apply Lemma~\ref{lem:warren} again:
  \begin{equation*}
    \abs{ \{ ( \sign f(\vx_1, \theta), \dots, \sign f(\vx_m, \theta)) \colon \theta \in S \}} \le 2
    \left(\frac{2em k d_L}{ W_L}\right)^{ W_L}.
  \end{equation*}
  By property 1, we can bound the size of $\cS_{L-1}$:
  \begin{equation*}
    \abs{\cS_{L}} \leq \prod_{l=1}^{L-1} 2\left(\frac{2emn_l p d_l}{ W_l} \right)^{ W_l}.
  \end{equation*}
  Combining the two bounds along with equation~\eqref{eqn:sum_partition} yields
  \begin{equation*}
    K  \le \prod_{l=1}^{L}  2\left(\frac{2empn_l d_l}{ W_l} \right)^{ W_l}.
  \end{equation*}
  We can take logarithm and apply Jensen's inequality, with $\bar W \defeq
  \sum_{l=1}^{L} W_l$:
  \begin{align*}
    \log_2 K
    &\leq L + \sum_{l=1}^{L} W_l \log_2 \frac{2 em p n_l d_l}{W_l} \\
    &= L + \bar W \sum_{l=1}^{L} \frac{W_l}{\bar W} \log_2 \frac{2 em p n_l d_l}{W_l} \\
    &\leq L + \bar W \log_2 \left( \sum_{l=1}^{L} \frac{W_l}{\bar W}\frac{2em p n_l
      d_l}{W_l} \right) && \text{(Jensen's inequality)}\\
    &= L + \bar{W} \log_2 \frac{2em p \sum_{l=1}^L n_l d_l}{\bar{W}}.
  \end{align*}
  We can bound $\sum n_l d_l$ using the bound on $d_l$ from
  Lemma~\ref{lem:degree_bound_nn}:
  \begin{equation*}
    \sum_{l=1}^{L} n_l d_l \leq \sum_{l=1}^{L} n_l c_1 k^{l-1} \sum_{j=0}^{l-1} n_j^{c_2} \leq L U c_1
    k^{L-1} U^{c_2} \leq c_1 U^{c_2+2} k^{L},
  \end{equation*}
  where we used the fact that $L \leq U$.
  Thus
  \begin{equation*}
    \log_2 K \leq L + \bar{W} \log_2 \frac{2 c_1 em p U^{2 + c_2} k^L}{\bar{W}}.
  \end{equation*}

  To bound the VC-dimension, recall that by definition, if $\VCdim(\sign \cF) =
  m$ then exists $m$ data points $\vx_1, \dots, \vx_m$ such that the output of the
  model can have $2^n$ sign patterns.
  The bound on $\log_2 K$ then implies
  \begin{equation*}
    \VCdim(\sign \mathcal{F}) \leq L + \bar{W} \log_2 \frac{2 c_1 e p U^{2+c_2} k^L
      \VCdim(\sign \mathcal{F})}{\bar{W}}.
  \end{equation*}
  We then use Lemma~\ref{lem:growthtovc} below, noting that $2 c_1 ep U^{2+c_2}
  k^L \geq 16$, to conclude that
  \begin{align*}
    \VCdim(\sign \cF) \le L + \bar{W} \log_2(2 c_1 ep U^{2+c_2} k^L \log_2 (2 c_1 ep U^{2+c_2} k^L))
    = O(\bar L W \log (pU) + \bar L L W \log k),
  \end{align*}
  completing the proof.

\end{proof}

A bound on the VC dimension immediate yields a bound on the sample complexity of
learning from this class of neural networks with LDR matrices~\cite{vapnik1998statistical}.
\begin{corollary}
  The class of neural network with LDR matrices as weights and piecewise linear
  activation is $(\epsilon, \delta)$-PAC-learnable with a sample of size
  \begin{equation*}
    O \left( \frac{L W \log W + \log \frac{1}{\delta}}{\epsilon} \right).
  \end{equation*}
\end{corollary}
Since the number of parameters $W$ is around the square root of the number of
parameters of a network with unstructured layers (assuming fixed rank
of the LDR matrices), the sample complexity of LDR networks is much smaller than
that of general unstructured networks.

\begin{lemma}[Lemma 16 of \cite{bartlett2017nearly}]
  \label{lem:growthtovc}
  Suppose that $2^m \leq 2^t (mr/w)^w$ for some $r\geq16$ and $m \geq w\geq t \geq 0$.
  Then, $m \leq t + w \log_2 (2r \log_2 r)$.
\end{lemma}

\paragraph{Extension to rational functions.}

We now show that Theorem~\ref{thm:vc_bound} holds for matrices where the entries are rational functions---rather than polynomials---of its parameters, incurring only a constant in the bound.
To define the function class $\sign \mathcal{F}$, we account for the possibility of poles by defining $\sign(a/0) = 0$.

We only need to check that Lemma~\ref{lem:degree_bound_nn} and Lemma~\ref{lem:warren} still hold when polynomials are replaced by rational functions everywhere, and the degree of a rational function is defined as the usual $\deg(a/b) = \max\{\deg a, \deg b\}$.
To show Lemma~\ref{lem:degree_bound_nn} still holds, it suffices that the compositional degree bound $\deg(f \circ g) \le \deg(f)\deg(g)$ holds for rational functions $f,g$, just as in the polynomial case.
To show Lemma~\ref{lem:warren} in the case when $p_i = a_i/b_i$ are rational functions, we note that $\sign(p_i(x)) = \sign(a_i(x)b_i(x))$, and furthermore $\deg(a_ib_i) \le 2\deg(p_i)$.
Appealing to the polynomial version of Lemma~\ref{lem:warren} shows that it holds in the rational function setting with a slightly weaker upper bound $K \le 2(4emd/n)^n$.
This gets converted to a constant factor in the result of Theorem~\ref{thm:vc_bound}.

Next, we extend Lemma~\ref{lem:degree_bound_ldr} by showing that generic LDR matrices have entries which are rational functions of their parameters.
This immediately lets us conclude that neural networks built from any LDR matrices satisfy the VC dimension bounds of Theorem~\ref{thm:vc_bound}.

\begin{lemma}
  If $\vM \in \mathbb{R}^{m \times m}$ satisfies $\vA\vM - \vM\vB = \vG\vH^T$,
  then the entries of $\vM$ are rational functions of the entries
  of $\vA, \vB, \vG, \vH$ with total degree at most $c_1 m^{c_2}$ for some universal
  constants $c_1, c_2 > 0$.
  \label{lem:degree_bound_gen}
\end{lemma}
\begin{proof}
  The vectorization of the Sylvester equation $\vA\vM-\vM\vB=\vR$ is $(\vI\otimes\vA - \vB^\top\otimes\vI)\vect(\vM) = \vect(\vR)$, where $\vect{}$ denotes the vectorization operation by stacking a matrix's columns, and $\otimes$ is the Kronecker product.
  Note that the entries of $\vN^{-1}$ for an arbitrary matrix $\vN \in \R^{n \times n}$
  are rational functions with degree $n$ in the entries of $\vN$,
  and $\vR=\vG\vH^\top$ has degree $2$ in the entries of $\vG,\vH$.
  Therefore the entries of
  \[
    \vect(\vM) = (\vI\otimes\vA - \vB^\top\otimes\vI)^{-1}\vect(\vR)
  \]
  have degree $n^2+2$ in the entries of $\vA,\vB,\vG,\vH$.
\end{proof}

Note that many other classes of matrices satisfy this lemma.
For example, a large class of matrices satisfying a property called \emph{low recurrence width}
was recently introduced as a way of generalizing many known structured matrices~\cite{desa2018two}.
The low recurrence width matrices are explicitly defined through a polynomial recurrence and satisfy the bounded degree condition.
Additionally, Lemma~\ref{lem:degree_bound_gen} holds when the parameters $\vA,\vB$ themselves are structured matrices
with entries having polynomial degree in terms of some parameters.
This includes the case when they are quasiseparable matrices, the most general class of LDR previously analyzed~\cite{desa2018two}.

\section{Additional results}
\label{sec:additional-results}

\subsection{Additional baselines and comparisons at multiple budgets}
In Tables~\ref{table:images-extended-shl} and~\ref{table:images-extended-cnn} we compare to baselines at parameter budgets corresponding to both the LDR-TD and LDR-SD classes in the SHL and CNN models.
In Tables~\ref{table:pruning-shl} and~\ref{table:pruning-cnn}, we also compare to two additional baselines, network pruning~\cite{han2015learning} and a baseline used in~\cite{chen2015compressing}, in which the number of hidden units is reduced to meet the parameter budget. We refer to this baseline as RHU ("reduced hidden units").
We show consistent improvements of LDR-SD over both methods at several budgets. We note that unlike the structured matrix methods which provide compression benefits during both training and inference, pruning requires first training the original model, followed by retraining with a fixed sparsity pattern.

\begin{table}[ht!]
  \centering
  \caption{Test accuracy when replacing the hidden layer with structured classes in the \textbf{single hidden layer} architecture, at parameter budgets corresponding to LDR-TD and LDR-SD rank one. Rank is in parentheses. The first group of structured methods (in orange) all have compression factors (relative to a general unstructured layer) of 98 on MNIST-bg-rot and MNIST-noise, and 128 on CIFAR-10 and NORB. The second group of structured methods (in blue) all have compression factors of 196 on MNIST-bg-rot and MNIST-noise, and 256 on CIFAR-10 and NORB.}
  \begin{tabular}{@{}lllll@{}}
    \toprule
    \textbf{Method} &  \textbf{MNIST-bg-rot} & \textbf{MNIST-noise} & \textbf{CIFAR-10} & \textbf{NORB}\\ \midrule
        Unstructured      & 44.08 & 65.15 & 46.03 &  59.83                 \\
        \hhline{=====}
   \cellcolor{orange!25}LDR-TD ($r=1$)  &    \cellcolor{orange!25}\textbf{45.81} & \cellcolor{orange!25}\textbf{78.45} & \cellcolor{orange!25}\textbf{45.33} & \cellcolor{orange!25}\textbf{62.75}                     \\
     \cellcolor{orange!25}Toeplitz-like ~\cite{sindhwani2015structured} ($r=4$)         & \cellcolor{orange!25}42.67 & \cellcolor{orange!25}75.75 & \cellcolor{orange!25}41.78 &  \cellcolor{orange!25}59.38                    \\ 
       \cellcolor{orange!25}Hankel-like ($r=4$)     &     \cellcolor{orange!25}42.23 & \cellcolor{orange!25}73.65 & \cellcolor{orange!25}41.40 & \cellcolor{orange!25}60.09                \\
    \cellcolor{orange!25}Vandermonde-like ($r=4$)      & \cellcolor{orange!25}37.14 & \cellcolor{orange!25}59.80 & \cellcolor{orange!25}33.93 & \cellcolor{orange!25}48.98                   \\
        \cellcolor{orange!25}Low-rank ~\cite{denil2013predicting} ($r=4$)     & \cellcolor{orange!25}35.67 & \cellcolor{orange!25}52.25 & \cellcolor{orange!25}32.28 & \cellcolor{orange!25}43.66                   \\
         \midrule
   \cellcolor{blue!25}LDR-SD ($r=1$)  &    \cellcolor{blue!25}\textbf{44.74} & \cellcolor{blue!25}\textbf{78.80} & \cellcolor{blue!25}\textbf{43.29} & \cellcolor{blue!25}\textbf{63.78}                     \\
     \cellcolor{blue!25}Toeplitz-like ~\cite{sindhwani2015structured} ($r=2$)         & \cellcolor{blue!25}42.07 & \cellcolor{blue!25}74.25 & \cellcolor{blue!25}40.68 & \cellcolor{blue!25}57.27                   \\ 
       \cellcolor{blue!25}Hankel-like ($r=2$)     &     \cellcolor{blue!25}41.01 & \cellcolor{blue!25}71.20 & \cellcolor{blue!25}40.46 & \cellcolor{blue!25}57.95                \\
      \cellcolor{blue!25}Vandermonde-like ($r=2$)      & \cellcolor{blue!25}33.56 & \cellcolor{blue!25}50.85 & \cellcolor{blue!25}28.99 & \cellcolor{blue!25}43.21                   \\
       \cellcolor{blue!25}Low-rank ~\cite{denil2013predicting} ($r=2$)     & \cellcolor{blue!25}32.64 & \cellcolor{blue!25}38.85 & \cellcolor{blue!25}24.93 & \cellcolor{blue!25}37.03                  \\
    \bottomrule
  \end{tabular}
  \label{table:images-extended-shl}
\end{table}

\begin{table}[ht!]
  \centering
  \caption{Test accuracy when replacing the fully-connected layer with structured classes in the \textbf{CNN} architecture, at parameter budgets corresponding to LDR-TD and LDR-SD rank one. Rank is in parentheses. The first group of structured methods (in orange) all have compression factors (relative to a general unstructured layer) of 98 on MNIST-bg-rot and MNIST-noise, and 128 on CIFAR-10 and NORB. The second group of structured methods (in blue) all have compression factors of 196 on MNIST-bg-rot and MNIST-noise, and 256 on CIFAR-10 and NORB.}
  \begin{tabular}{@{}lllll@{}}
    \toprule
    \textbf{Method} &  \textbf{MNIST-bg-rot} & \textbf{MNIST-noise} & \textbf{CIFAR-10} & \textbf{NORB}\\ \midrule
        Fully-connected      & 67.94 & 90.30 & 68.09 & 75.16                \\
        \hhline{=====}
\cellcolor{orange!25}LDR-TD ($r=1$)  &    \cellcolor{orange!25}\textbf{68.79} & \cellcolor{orange!25}\textbf{92.55} & \cellcolor{orange!25}66.63 & \cellcolor{orange!25}\textbf{74.23}                     \\
     \cellcolor{orange!25}Toeplitz-like~\cite{sindhwani2015structured} ($r=4$)         & \cellcolor{orange!25}63.23 & \cellcolor{orange!25}91.60 & \cellcolor{orange!25}67.10 & \cellcolor{orange!25}72.25                   \\ 
       \cellcolor{orange!25}Hankel-like ($r=4$)     &     \cellcolor{orange!25}64.21 & \cellcolor{orange!25}90.80 & \cellcolor{orange!25}\textbf{68.10} & \cellcolor{orange!25}71.23                \\
   \cellcolor{orange!25}Vandermonde-like ($r=4$)      & \cellcolor{orange!25}61.76 & \cellcolor{orange!25}90.40 & \cellcolor{orange!25}63.63 & \cellcolor{orange!25}72.11                   \\
       \cellcolor{orange!25}Low-rank~\cite{denil2013predicting} ($r=4$)     & \cellcolor{orange!25}60.35 & \cellcolor{orange!25}87.30 & \cellcolor{orange!25}60.90 & \cellcolor{orange!25}71.47                   \\
         \midrule
   \cellcolor{blue!25}LDR-SD ($r=1$)  &    \cellcolor{blue!25}\textbf{67.40} & \cellcolor{blue!25}\textbf{92.20} &  \cellcolor{blue!25}65.48 & \cellcolor{blue!25}\textbf{73.63}                     \\
     \cellcolor{blue!25}Toeplitz-like~\cite{sindhwani2015structured} ($r=2$)         & \cellcolor{blue!25}63.63 & \cellcolor{blue!25}91.45 & \cellcolor{blue!25}67.15 & \cellcolor{blue!25}71.64                    \\ 
       \cellcolor{blue!25}Hankel-like ($r=2$)     &     \cellcolor{blue!25}64.08 & \cellcolor{blue!25}90.65 & \cellcolor{blue!25}\textbf{67.49} & \cellcolor{blue!25}71.21                \\
      \cellcolor{blue!25}Vandermonde-like ($r=2$)      & \cellcolor{blue!25}51.38 & \cellcolor{blue!25}86.50 & \cellcolor{blue!25}58.00 & \cellcolor{blue!25}68.08                   \\
      \cellcolor{blue!25}Low-rank~\cite{denil2013predicting} ($r=2$)     & \cellcolor{blue!25}41.91 & \cellcolor{blue!25}71.15 & \cellcolor{blue!25}48.48 & \cellcolor{blue!25}65.34                   \\
    \bottomrule
  \end{tabular}
  \label{table:images-extended-cnn}
\end{table}

\begin{table}[htbp]
\centering

\caption{On the MNIST variants, in the \textbf{single hidden layer} architecture, we compare LDR-SD, pruning~\cite{han2015learning}, and a baseline which reduces the number of hidden units (denoted RHU), at multiple budgets. At each budget, we adjust the number of pruned weights or hidden units to match as closely as possible the parameter budget of LDR-SD. Parameter counts of fully-connected layers for LDR-SD and pruning at ranks 1,2,4,8,12, and 16 are 10986, 12554, 15690, 21962, 28234, and 34506 respectively, and 11126, 12714, 15890, 22242, 28594, 34946 for RHU (for which parameter count cannot be controlled exactly). As shown above, we find that the classification accuracy of LDR-SD consistently exceeds that of both methods.}
\begin{subtable}{.5\textwidth}
\centering
    \begin{tabular}{@{}lllll@{}}
    \toprule
    \textbf{Rank of LDR-SD} &  LDR-SD & Pruning~\cite{han2015learning} & RHU~\cite{chen2015compressing}\\ \midrule
       1    & \textbf{44.74}  & 40.41  & 37.18      \\
  2  &   \textbf{44.46}      & 41.18    & 37.60    \\
   4 &  \textbf{47.72}       & 42.45   & 37.98     \\
   8 &  \textbf{48.76}          & 43.52  & 39.77  \\
   12 &   \textbf{48.90}      &  43.19   & 40.56    \\
   16 &   \textbf{49.51}      & 43.58    & 40.70    \\
    \bottomrule
  \end{tabular}  
\caption{MNIST-bg-rot}
\end{subtable}\hfil \\
\vspace{1em}
\begin{subtable}{.5\textwidth}
\centering
\begin{tabular}{@{}lllll@{}}
    \toprule
  \textbf{Rank of LDR-SD} &  LDR-SD          & Pruning~\cite{han2015learning} & RHU~\cite{chen2015compressing}\\ \midrule
  1                       & \textbf{78.80}   & 67.75                          & 62.85      \\
  2                       &   \textbf{77.95} & 69.35                          & 62.55       \\
  4                       &  \textbf{78.32}  & 68.25                          & 63.40      \\
  8                       &  \textbf{78.63}  & 67.25                          & 64.45   \\
  12                      &   \textbf{78.33} &  67.30                         & 63.85      \\
  16                      &   \textbf{78.08} & 66.95                          & 66.10      \\
    \bottomrule
\end{tabular}
\caption{MNIST-noise}
\end{subtable}
  \label{table:pruning-shl}
\end{table}

\begin{table}[htbp]
\centering
  \caption{On the MNIST variants, in the \textbf{CNN} architecture, we compare LDR-SD, pruning~\cite{han2015learning}, and a baseline which reduces the number of hidden units (denoted RHU), at multiple budgets. At each budget, we adjust the number of pruned weights or hidden units to match as closely as possible the parameter budget of LDR-SD. Parameter counts of fully-connected layers for LDR-SD and pruning at ranks 1,2,4,8,12, and 16 are 11770, 13338, 16474, 22746, 29018, and 35290 respectively, and 11935, 13525, 16705, 23065, 29425, 35785 for RHU (for which parameter count cannot be controlled exactly). 
 As shown above, we find that the classification accuracy of LDR-SD consistently exceeds that of both methods.}
\begin{subtable}{.5\textwidth}
\centering
    \begin{tabular}{@{}lllll@{}}
    \toprule
    \textbf{Rank of LDR-SD} &  LDR-SD & Pruning~\cite{han2015learning} & RHU~\cite{chen2015compressing}\\ \midrule
       1    & \textbf{67.40} & 64.25     & 64.03     \\
  2  &   \textbf{67.53}      & 64.05    & 64.67      \\
   4 &  \textbf{67.96}       & 65.50    & 66.37     \\
   8 &  \textbf{67.21}          & 64.12   & 64.70   \\
   12 &   \textbf{68.54}      &  65.65      & 65.99    \\
   16 &   \textbf{67.00}      & 65.59      & 66.47     \\
    \bottomrule
  \end{tabular} 
\caption{MNIST-bg-rot}
\end{subtable}\hfil \\
\vspace{1em}
\begin{subtable}{.5\textwidth}
\centering
  \begin{tabular}{@{}lllll@{}}
    \toprule
    \textbf{Rank of LDR-SD} &  LDR-SD & Pruning~\cite{han2015learning} & RHU~\cite{chen2015compressing}\\ \midrule
       1    & \textbf{92.20} & 90.80     & 90.95     \\
  2  &   \textbf{92.75}      & 91.65     & 91.00     \\
   4 &  \textbf{91.30}       & 90.60      & 91.25    \\
   8 &  \textbf{91.95}          & 91.05    & 90.65  \\
   12 &   \textbf{92.10}      &  90.00      & 90.85    \\
   16 &   \textbf{93.20}      & 90.55       & 90.40    \\
    \bottomrule
  \end{tabular} 
\caption{MNIST-noise}
\end{subtable}
  \label{table:pruning-cnn}
\end{table}

\subsection{Sample complexity and generalization}
As shown in Tables~\ref{table:sample-complexity-shl} and~\ref{table:sample-complexity-cnn}, we investigated how the performance of the structured and general unstructured fully-connected layers varied with the amount of training data. On the MNIST variants, we trained both the single hidden layer and CNN models with random subsamples of 25\%, 50\%, and 75\% of the training set, with 15\% of the training set used for validation in all settings. 
In addition, in Table~\ref{table:gen-error}, we compare the generalization error of structured classes with an unstructured model, and find that the structured classes have consistently lower generalization error.

\begin{table}[htb]
\centering

  \caption{On the MNIST variants, in the \textbf{single hidden layer} architecture, we show how the number of training samples affects the performance of the unstructured model and the structured classes. Columns correspond to models trained on 25\%, 50\%, 75\% and 100\% of the training data (randomly subsampled). LDR-TD and LDR-SD consistently outperform the structured baselines at the tested subsampling ratios. On MNIST-bg-rot, LDR-TD only needs 75\% of the training data to outperform the unstructured model trained on 100\% of the training data. On MNIST-noise, both LDR-TD and LDR-SD only need 25\% of the training data to outperform the unstructured layer. All are rank one.}

\begin{subtable}{.5\textwidth}
\centering
  \begin{tabular}{@{}lllll@{}}
    \toprule
    \textbf{Method} &  \textbf{25\%} & \textbf{50\%} & \textbf{75\%} & \textbf{100\%}\\ \midrule
       Unstructured     & 34.46 & 38.80 & 43.35 & 44.08          \\
        \midrule
   LDR-TD  &    34.01  & 39.59 & \textbf{44.35} & \textbf{45.81}           \\
   LDR-SD &   \textbf{35.64} & \textbf{39.78} & 42.72 & 44.74                 \\
     Toeplitz-like         & 33.71 & 36.44 & 39.32 & 41.12         \\ 
        Low-rank      & 21.44 & 23.46 & 23.48 & 25.06                \\
    \bottomrule
  \end{tabular}
\caption{MNIST-bg-rot}
\label{subtable:bgrot}
\end{subtable}\hfil
\begin{subtable}{.5\textwidth}
\centering
  \begin{tabular}{@{}lllll@{}}
    \toprule
    \textbf{Method} &  \textbf{25\%} & \textbf{50\%} & \textbf{75\%} & \textbf{100\%}\\ \midrule
        Unstructured     & 59.30 & 61.85 & 65.35 & 65.15             \\
        \midrule
  LDR-TD  &    65.45 & \textbf{74.60} & \textbf{77.45} & 78.45               \\
   LDR-SD &   \textbf{67.90} & 71.15 & 76.95 & \textbf{78.80}                 \\
     Toeplitz-like         & 56.15 & 67.75 & 72.30 &73.95          \\ 
        Low-rank      & 24.25 & 26.20 & 26.85 & 26.40                  \\
    \bottomrule
  \end{tabular}
\caption{MNIST-noise}
\label{subtable:noise}
\end{subtable}
  \label{table:sample-complexity-shl}
\end{table}

\begin{table}[htb]
\centering
  \caption{On the MNIST variants, in the \textbf{CNN} architecture, we show how the number of training samples affects the performance of the unstructured model and the structured classes. Columns correspond to models trained on 25\%, 50\%, 75\% and 100\% of the training data (randomly subsampled). LDR-TD and LDR-SD consistently outperform the structured baselines at the tested subsampling ratios. On MNIST-noise, both LDR-TD and LDR-SD only need 50\% of the training data to outperform the unstructured layer. All are rank one.}
\begin{subtable}{.5\textwidth}
\centering
  \begin{tabular}{@{}lllll@{}}
    \toprule
    \textbf{Method} &  \textbf{25\%} & \textbf{50\%} & \textbf{75\%} & \textbf{100\%}\\ \midrule
       Unstructured    & 54.12 & 62.53 & 67.52 & 67.94            \\
        \midrule
   LDR-TD  &    \textbf{53.66} & \textbf{62.15} & \textbf{67.25} & \textbf{68.79}               \\
   LDR-SD &   50.72 & 61.92 & 65.93 & 67.40                  \\
     Toeplitz-like         & 49.10 & 57.20 & 61.53 & 63.00          \\ 
        Low-rank      & 26.98 & 27.97 & 28.97 & 29.63                  \\
    \bottomrule
  \end{tabular}
\caption{MNIST-bg-rot}
\end{subtable}\hfil
\begin{subtable}{.5\textwidth}
\centering
  \begin{tabular}{@{}lllll@{}}
    \toprule
    \textbf{Method} &  \textbf{25\%} & \textbf{50\%} & \textbf{75\%} & \textbf{100\%}\\ \midrule
      Unstructured    & 81.85 & 88.25 & 89.75 & 90.30            \\
        \midrule
   LDR-TD  &    86.45  & \textbf{91.35} & \textbf{93.00} & \textbf{92.55}               \\
   LDR-SD &   \textbf{86.95} & 90.90 & 91.55 & 92.20                 \\
     Toeplitz-like         & 81.65 & 88.15 & 90.90 & 90.95         \\ 
        Low-rank      & 33.15 & 38.40 & 42.55 & 44.55                 \\
    \bottomrule
  \end{tabular}
\caption{MNIST-noise}
\end{subtable}

  \label{table:sample-complexity-cnn}
\end{table}

\begin{table}[ht!]
  \centering
  \caption{Generalization error for unstructured, LDR-TD, LDR-SD, Toeplitz-like, low-rank classes on the single hidden layer architecture. Consistent with Theorem~\ref{thm:vc}, the structured classes have consistently lower generalization error than the unstructured model. All are rank one.}
  \begin{tabular}{@{}lllll@{}}
    \toprule
    \textbf{Method} &  \textbf{MNIST-bg-rot} & \textbf{MNIST-noise} & \textbf{CIFAR-10} & \textbf{NORB}\\
    \midrule
        Unstructured      & 55.78 & 21.63 & 34.32 & 40.03             \\
    \midrule
   LDR-TD &  13.52 & 11.36 & 7.10 & 9.51                \\
   LDR-SD &   12.87 & 12.65 & 6.29 & 8.68                \\
     Toeplitz-like~\cite{sindhwani2015structured}         & 7.98 & 15.80 & 5.59 & 7.87         \\ 
        Low-rank~\cite{denil2013predicting}      & 8.40 & 0.31 & 0.09 & 2.59                  \\
    \bottomrule
  \end{tabular}
  \label{table:gen-error}
\end{table}

\subsection{Additional visualizations}
In Figure~\ref{fig:visualization-NORB}, we visualize the learned subdiagonal on NORB along with images from the dataset.

\begin{figure}[ht!]
  \centering
  \begin{subfigure}{0.5\linewidth}
    \centering
    \includegraphics[width=0.5\linewidth]{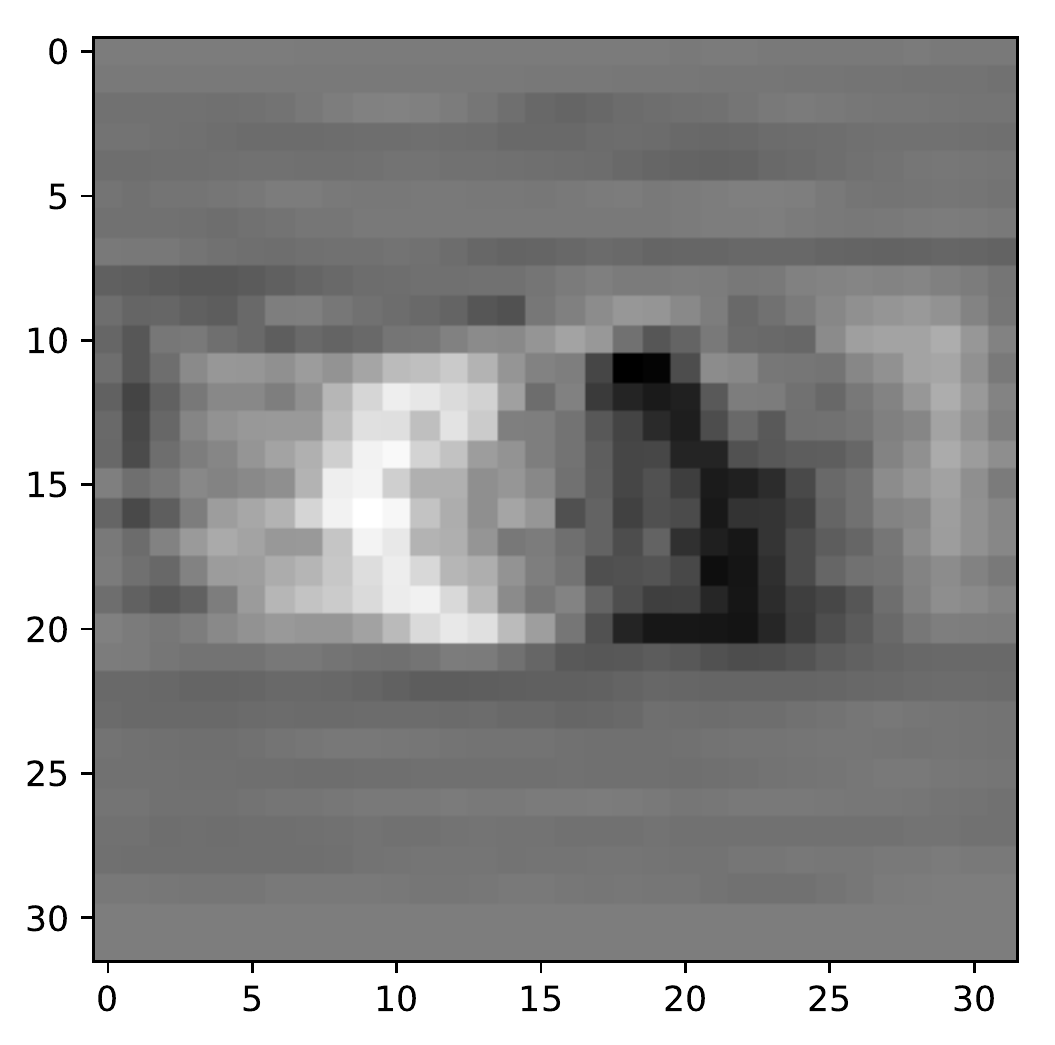}
    \caption{Subdiagonal of $\vB$ (NORB)}
  \end{subfigure}\hfill
  \begin{subfigure}{0.5\linewidth}
    \centering
    \includegraphics[width=0.5\linewidth]{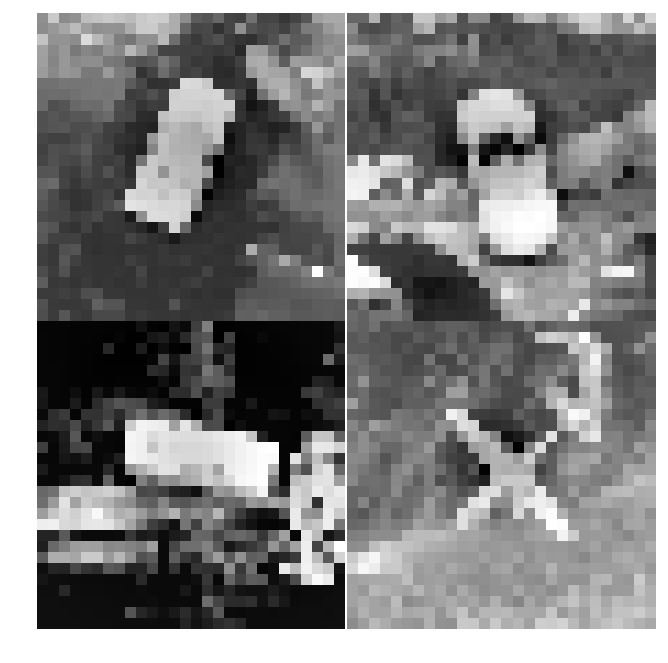}
    \caption{Images from NORB}
  \end{subfigure}\hfill
  \caption{We visualize the learned subdiagonal of the operator $\vB$ and images from the NORB dataset. We observe a centering phenomenon similar to that described in Figure~\ref{fig:visualization}.} 
  \label{fig:visualization-NORB}
\end{figure}

On the MNIST-bg-rot dataset~\cite{larochelle2007empirical}, we note that~\citet{chen2015compressing} also tested several methods on this dataset, including Random Edge Removal~\cite{cirecsan2011high}, Low Rank Decomposition~\cite{denil2013predicting}, Dark Knowledge~\cite{hinton2015distilling}, HashedNets~\cite{chen2015compressing}, and HashedNets with Dark Knowledge, and reported test errors of  73.17, 80.63, 79.03, 77.40, 59.20, and 58.25, where each method had 12406 parameters in the architecture. We found that our LDR-SD class, with 10986 parameters in the architecture, achieved a test error of 55.26, as shown in Table~\ref{table:images-extended-shl}, outperforming all methods evaluated by~\citet{chen2015compressing}.~\citet{sindhwani2015structured} later also tested on this dataset, and reported test errors of 68.4, 62.11, and 55.21 for Fastfood (10202 parameters), Circulant (8634 parameters), and Toeplitz-like, $r=2$ (10986 parameters).  LDR-SD exceeds their reported results for Fastfood and Circulant~\cite{cheng2015exploration}, but not that of Toeplitz-like. We did find that our proposed classes consistently exceeded the performance of our own implementation of Toeplitz-like on this dataset (Table~\ref{table:images}, Figure~\ref{fig:rank-vs-accuracy}, and Tables~\ref{table:images-extended-shl} and~\ref{table:images-extended-cnn}).

\subsection{Rectangles dataset}

We provide an interesting example of a case where LDR-TD and LDR-SD do not exceed the performance of the fixed operator classes in the single hidden layer architecture.
In this simple dataset from~\citet{larochelle2007empirical},
the task is to classify a binary image of a rectangle as having a greater length or width.
We show examples of the dataset in Figure~\ref{fig:rectangles}.
On this dataset, in contrast to the more challenging datasets (MNIST-bg-rot~\citep{larochelle2007empirical}, MNIST-noise~\citep{larochelle2007empirical}, CIFAR-10~\citep{krizhevsky2009learning}, and NORB~\citep{lecun2004learning}) we tested on, every structured class outperforms an unconstrained model (622506 parameters), including the circulant class~\cite{cheng2015exploration} which compresses the hidden layer by 784x, and expanding the class beyond Toeplitz-like does not improve performance.
We hypothesize that this is because the Toeplitz-like class may enforce the right structure, in the sense that it is sufficiently expressive to fit a perfect model on this dataset, but not expansive enough to lead to overfitting.
For example, while the Toeplitz-like operators model approximate shift equivariance (discussed in Section~\ref{sec:equivariance} and Proposition~\ref{prop:equivariance} in Section~\ref{sec:group_rep}),
the additional scaling that subdiagonal operators provide is unnecessary on these binary inputs.
\begin{figure}[h!]
  \centering
\includegraphics[width=\textwidth]{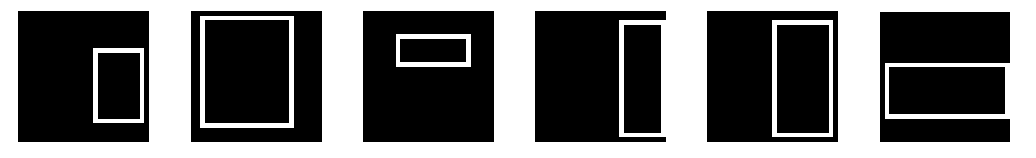}
  \caption{Examples of images from the rectangles dataset~\cite{larochelle2007empirical}.}
  \label{fig:rectangles}
\end{figure}

\begin{table}[ht!]
  \centering
  \caption{Test accuracy when replacing the hidden layer with structured classes on the rectangles dataset~\cite{larochelle2007empirical}. Where applicable, rank ($r$) is in parentheses, and the number of parameters in the architecture is in italics below each method.}
  \begin{tabular}{@{}lllll@{}}
    \toprule
    \textbf{Method} & Test Accuracy\\ \midrule
        Unconstrained      & 91.94               \\
                 & \textit{\textcolor{gray}{622506}}                     \\ 
        \hhline{=====}
   LDR-TD ($r=1$)  &    98.53                  \\
                 & \textit{\textcolor{gray}{14122}}                     \\ 
                 \midrule
   LDR-SD ($r=1$)  &    98.39                  \\
                 & \textit{\textcolor{gray}{10986}}    \\   
	\midrule
     Toeplitz-like ($r=4$)~\cite{sindhwani2015structured}         & \textbf{99.29}          \\ 
                 & \textit{\textcolor{gray}{14122}}                     \\ 
	\midrule
       Hankel-like ($r=4$)     &     97.77           \\
                 & \textit{\textcolor{gray}{14122}}                     \\ 
 	\midrule
    Vandermonde-like ($r=4$)      & 94.11             \\
                 & \textit{\textcolor{gray}{14122}}                     \\ 
 	\midrule
        Low-rank ($r=4$)~\cite{denil2013predicting}     & 92.80                   \\
                 & \textit{\textcolor{gray}{14122}}                     \\ 
	\midrule
      Fastfood~\cite{yang2015deep}     & 92.20                 \\
                 & \textit{\textcolor{gray}{10202}}                     \\ 
	\midrule
      Circulant~\cite{cheng2015exploration}     & 95.58                    \\
                 & \textit{\textcolor{gray}{8634}}                     \\ 
    \bottomrule
  \end{tabular}
  \label{table:rectangle-dataset}
\end{table}

\subsection{Acceleration at inference time}
\label{sec:speed}
We empirically study the acceleration obtained at inference time (on CPU) with our \href{https://github.com/HazyResearch/structured-nets}{implementation} of the algorithms for multiplication by LDR-SD described in Appendix~\ref{sec:algo}. We generated random parameters for each class and ran each multiplication algorithm 1000 times to compare the speedup of each class over an unstructured multiply. Each test was repeated 10 times, and the minimum total runtime over the 10 tests was used for each class.
As shown in Figure~\ref{fig:speed} and Table~\ref{table:speed}, at $n \geq 4096$, our simple Python implementation is 3.34-46.06x faster than the highly optimized unstructured matrix-vector multiply (a BLAS level 2 operation). We also compare with two other structured classes, low-rank and Toeplitz-like, at $r=1,2,4,8,16$. A batch size of one was used in all tests. The time complexity of multiplication by low-rank and Toeplitz-like is $O(nr)$ and $O(nr \log n)$ respectively, compared to $O(nr \log^2 n)$ for LDR-SD.

\begin{table}[htbp]
\centering
\caption{Acceleration of $n \times n$ structured classes over unstructured matrix-vector multiply at inference time. Experimental details are in Appendix~\ref{sec:speed}.}
\begin{subtable}{\textwidth}
  \centering
  \begin{tabular}{ c | c | c | c | c | c }
    \toprule
     &  \multicolumn{5}{c}{Rank} \\ \midrule
     $n$ &  1 & 2 & 4 & 8 & 16\\ \midrule
       $2^9$    & $5.15 \times 10^1$ & $2.43 \times 10^1$ & $2.46 \times 10^1$ & $2.08 \times 10^1$  & $1.81 \times 10^1$     \\
       $2^{10}$    & $1.39 \times 10^2$ & $5.41 \times 10^1$ &  $5.66 \times 10^1$ & $4.62 \times 10^1$   & $3.43 \times 10^1$         \\
       $2^{11}$    & $4.14 \times 10^2$ & $1.60 \times 10^2$ & $1.71 \times 10^2$ & $1.05 \times 10^2$   & $6.90 \times 10^1$          \\
       $2^{12}$    & $2.38 \times 10^3$ & $8.71 \times 10^2$ & $7.46 \times 10^2$  & $4.73 \times 10^2$   & $3.59 \times 10^2$           \\
       $2^{13}$   & $5.96 \times 10^3$ & $1.75 \times 10^3$  & $1.65 \times 10^3$ & $1.13 \times 10^3$    & $8.86 \times 10^2$         \\
       $2^{14}$    & $8.35 \times 10^3$ & $3.44 \times 10^3$  & $3.40 \times 10^3$ & $2.29 \times 10^3$   & $1.74 \times 10^3$          \\
       $2^{15}$    & $1.79 \times 10^4$ &$7.50 \times 10^3$ & $7.53 \times 10^3$ & $4.91 \times 10^3$   & $3.70 \times 10^3$          \\
    \bottomrule
      \end{tabular}
  \caption{Low-rank}\hspace{10em}
\end{subtable}
\begin{subtable}{\textwidth}
  \centering
  \begin{tabular}{ c | c | c | c | c | c }
    \toprule
     &  \multicolumn{5}{c}{Rank} \\ \midrule
     $n$ &  1 & 2 & 4 & 8 & 16\\ \midrule
       $2^9$    & $3.06 \times 10^{-1}$ & $2.60 \times 10^{-1}$ & $2.32 \times 10^{-1}$ & $1.86 \times 10^{-1}$  & $1.61 \times 10^{-1}$     \\
       $2^{10}$    & $7.34 \times 10^{-1}$  & $6.21 \times 10^{-1}$ & $5.18 \times 10^{-1}$ & $4.00 \times 10^{-1}$   & $3.28 \times 10^{-1}$          \\
       $2^{11}$    & $1.90 \times 10^{0}$  & $1.71 \times 10^{0}$ & $1.38 \times 10^{0}$ & $1.08 \times 10^{0}$   & $8.46 \times 10^{-1}$          \\
       $2^{12}$    & $1.23 \times 10^{1}$ & $1.01 \times 10^{1}$ & $7.92 \times 10^{0}$ & $5.97 \times 10^{0}$   & $4.62 \times 10^{0}$          \\
       $2^{13}$   & $3.34 \times 10^{1}$ & $2.73 \times 10^{1}$ & $2.26 \times 10^{1}$ & $1.52 \times 10^{1}$    & $1.23 \times 10^{1}$         \\
       $2^{14}$    & $6.96 \times 10^{1}$ & $5.68 \times 10^{1}$ & $4.19 \times 10^{1}$ &  $3.00 \times 10^{1}$   & $2.26 \times 10^{1}$          \\
       $2^{15}$    & $1.49 \times 10^{2}$ & $1.19 \times 10^{2}$ & $9.07 \times 10^{1}$ & $5.46 \times 10^{1}$   & $3.82 \times 10^{1}$          \\
    \bottomrule
      \end{tabular}
  \caption{Toeplitz-like}\hspace{10em}
\end{subtable}
\begin{subtable}{\textwidth}
  \centering
  \begin{tabular}{ c | c | c | c | c | c }
    \toprule
     &  \multicolumn{5}{c}{Rank} \\ \midrule
    $n$ &  1 & 2 & 4 & 8 & 16\\ \midrule
       $2^9$    & $6.68 \times 10^{-2}$ & $4.63 \times 10^{-2}$ & $4.05 \times 10^{-2}$& $3.10 \times 10^{-2}$  & $2.56 \times 10^{-2}$     \\
       $2^{10}$    & $1.49 \times 10^{-1}$ & $1.20 \times 10^{-1}$ & $9.45 \times 10^{-2}$ & $6.73 \times 10^{-2}$   & $5.24 \times 10^{-2}$          \\
       $2^{11}$    & $4.99 \times 10^{-1}$ & $4.32 \times 10^{-1}$ & $3.02 \times 10^{-1}$ & $1.94 \times 10^{-1}$   & $1.37 \times 10^{-1}$          \\
       $2^{12}$    & $3.34 \times 10^{0}$ & $2.57 \times 10^{0}$ & $1.61 \times 10^{0}$  & $1.06 \times 10^{0}$   & $7.52 \times 10^{-1}$            \\
       $2^{13}$   & $9.71 \times 10^{0}$ & $6.61 \times 10^{0}$ & $4.40 \times 10^{0}$ & $2.46 \times 10^{0}$    & $1.68 \times 10^{0}$          \\
       $2^{14}$    & $2.12 \times 10^{1}$ &$1.41 \times 10^{1}$  & $8.38 \times 10^{0}$ & $4.35 \times 10^{0}$   & $3.00 \times 10^{0}$          \\
       $2^{15}$    & $4.61 \times 10^{1}$ & $2.82 \times 10^{1}$ & $1.60 \times 10^{1}$ & $8.58 \times 10^{0}$   & $5.70 \times 10^{0}$          \\
    \bottomrule
      \end{tabular}
  \caption{LDR-SD}
\end{subtable}\hfil
  \label{table:speed}
\end{table}

\begin{figure}[h!]
  \centering
  \includegraphics[width=0.4\textwidth]{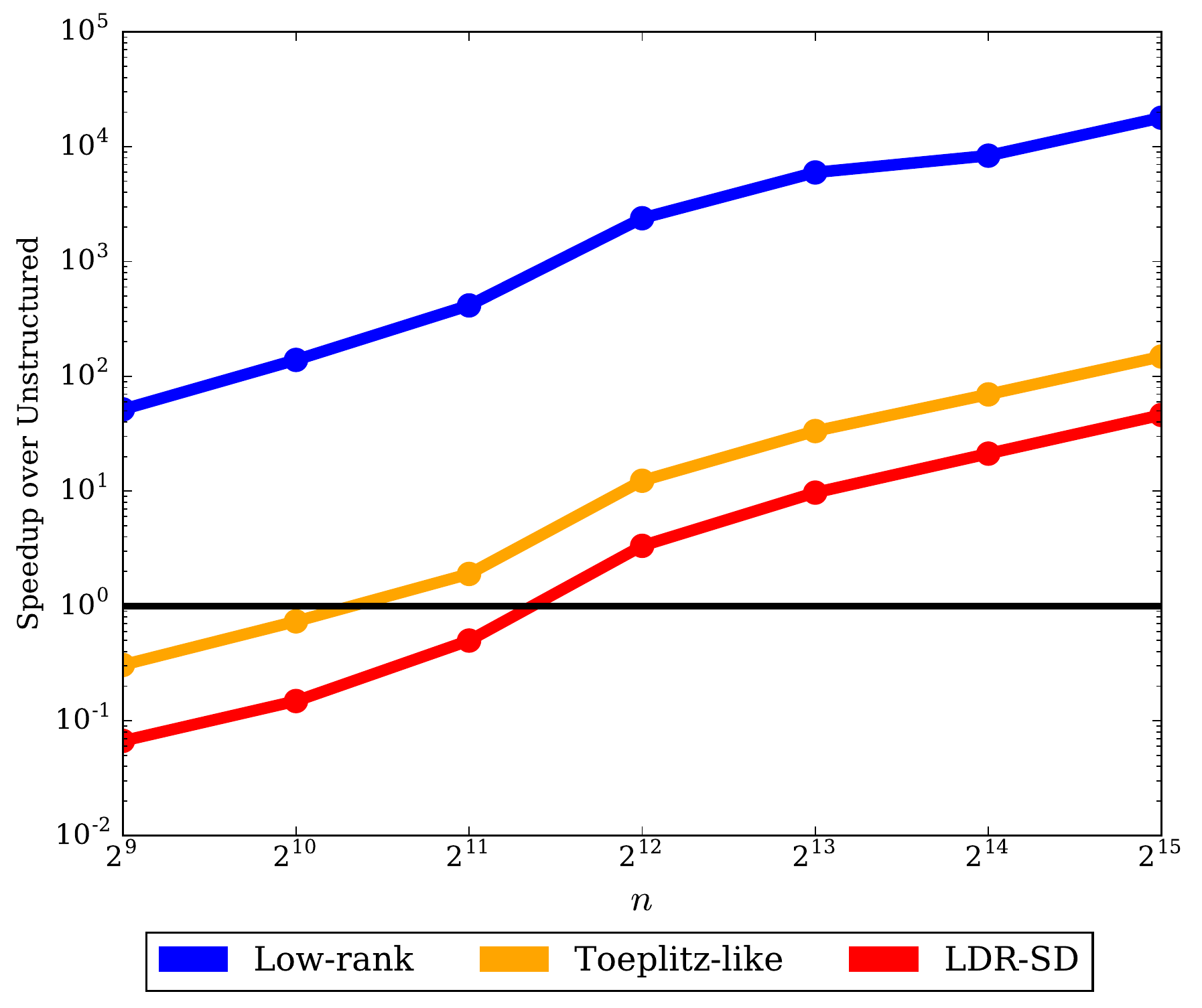}
  \caption{Acceleration of $n \times n$  structured classes over unstructured matrix-vector multiply at inference time. At $n \geq 4096$, LDR-SD ($r=1$) achieves a speedup of 3.34-46.06x over unstructured. Data for higher ranks are shown in Table~\ref{table:speed}. The comparison to the low-rank and Toeplitz-like classes illustrates a tradeoff involved in broadening the class of structured matrices we learn over. Though LDR-SD consistently outperforms these classes on downstream quality, its computational cost of multiplication is $O(nr \log^2 n)$, compared to $O(nr)$ and $O(nr \log n)$ for low-rank and Toeplitz-like respectively. Experimental details are in Appendix~\ref{sec:speed}.}
  \label{fig:speed}
\end{figure}

\section{Experimental details}
\label{sec:exp-details}

\subsection{Image classification}

In Table~\ref{table:datasets}, we provide details on the datasets we use for evaluation. For all our experiments, batch sizes were chosen to be 50. NORB was downsampled to $32\times32$, and the left stereo image was used. Training was performed with stochastic gradient descent with momentum, with the number of epochs set to 50 on all datasets. 15\% of the training data was used for the validation set on all experiments. We fixed momentum at 0.9 for all methods for all experiments, and performed a grid search over learning rate. Unless otherwise stated, for each method, we tested the learning rates \{0.0002, 0.0005, 0.001, 0.002\}, with three trials (with random initializations) per learning rate. For each trial, we test on the validation set at each epoch, and report the test accuracy of the model with the highest validation accuracy, over all learning rates, trials, and epochs.

In Figure~\ref{fig:rank-vs-accuracy}, for each method and each of the four learning rates, we perform five trials with random initializations and report the average and standard deviation of the test accuracy of the learning rate with the highest average validation accuracy.

\begin{table}[ht!]
  \centering
  \caption{Overview of the image classification datasets used in this work. For all datasets, 15\% of the training set was used for the validation set.}
  \begin{tabular}{@{}lllll@{}}
    \toprule
    \textbf{Dataset} & Training Examples & Test Examples & Number of Classes\\ \midrule
        MNIST-bg-rot~\cite{larochelle2007empirical}      & 12000 & 50000 & 10               \\
   MNIST-noise~\cite{larochelle2007empirical}  &    12000 & 2000 & 10                   \\
   CIFAR-10~\cite{krizhevsky2009learning}  &    50000 & 10000 & 10                  \\
   NORB~\cite{lecun2004learning}  &    291600 & 58320 & 6                  \\
   Rectangles~\cite{larochelle2007empirical}  &    1200 & 50000 & 2                   \\
    \bottomrule
  \end{tabular}
  \label{table:datasets}
\end{table}

\paragraph{Single hidden layer architecture}
In these experiments, we used an architecture consisting of a fully-connected hidden layer, followed by a fully-connected softmax layer. In order to be consistent with the architecture used in~\citet{sindhwani2015structured}, we do not use a bias term in the hidden layer.

\paragraph{CNN architecture}
In these experiments, shown in Table~\ref{table:images-extended-cnn} in Appendix~\ref{sec:additional-results}, we tested on a LeNet-based architecture. The architecture has 2 convolution/pool layers with 6 and 16 channels respectively, followed by a fully-connected layer, followed by fully-connected logit/softmax layer. We replaced the second to last fully-connected layer, which was of dimensions $784\times784$ for the MNIST-bg-rot and MNIST-noise datasets, and $1024\times1024$ for the CIFAR-10 and NORB experiments.

\paragraph{Replacing convolutional layers}

This experiment corresponds to Table~\ref{table:conv-filter}.

Here, we investigated whether the convolutional layers of CNNs can be learned automatically. 
For our experiments, we test on the simplest possible multi-channel CNN model on the CIFAR-10 dataset.
The model consists of one layer of convolutional channels ($3$ RGB in channels, $3$ out channels, stride $5$), followed by a fully-connected layer and a final FC+softmax layer (total of 4 layers).
We replace the convolutions with various structured matrices of the same dimensions, keeping the same $3 \times 3$ channel structure (e.g.\ it would consist of $3 \cdot 3 = 9$ square structured matrices) and number of hidden units.\footnote{The convolutions are padded to ensure their input and output dimensions are equal.}

The LDR classes benefit from being composed with LDR matrices of the same type (due to the composition property, Proposition~\ref{prop:closure}\ref{prop:closure:product}), so we additionally replace the later FC layer with the same structured matrix type.

By Proposition~\ref{prop:closure}\ref{prop:closure:block}, channels of Toeplitz-like matrices form a larger Toeplitz-like matrix of the same size.
Using this insight, we consider replacing the channel structure of the convolutional layer with either channels of structured matrices or a single wide structured matrix.
(Also, note that this is able to leverage the asymptotic fast nature of our structured classes.)

Because it seems that convolutional layers are strongly dependent on pooling -- our structured matrices outperform them in isolation -- we compare against a version of the CNN with an additional pooling layer after the convolutional channels.
Note that this comparison is the same basic four layer model with a structured matrix vs.\ a five layer convolutional model with pooling.
Since the architectures are quite different and difficult to directly compare, we also experimented with adding more hidden units to the pooling model.

\subsection{Language modeling}

For a language modeling application\footnote{Code available at \url{https://github.com/pytorch/examples/tree/master/word_language_model}.}, we explored replacing weight matrices in a recurrent neural network with structured matrices. We evaluate on a single layer LSTM architecture, defined by the update equations:

\begin{align*}
        i &= \sigma(W_{ii} x + b_{ii} + W_{hi} h + b_{hi}) \\
        f &= \sigma(W_{if} x + b_{if} + W_{hf} h + b_{hf}) \\
        g &= \tanh(W_{ig} x + b_{ig} + W_{hg} h + b_{hg}) \\
        o &= \sigma(W_{io} x + b_{io} + W_{ho} h + b_{ho}) \\
        c' &= f * c + i * g \\
        h' &= o \tanh(c') \\
\end{align*}

In our experiments we replace the matrices $W_{ii}, W_{if}, W_{ig}, W_{io}$ with structured matrices. We use a hidden layer of size 128, and word embedding size of 128. We evaluate on the Wikitext-2  dataset, which consists of Wikipedia articles (2,088,628 training, 217,646 validation, and 245,569 test tokens). The total vocabulary is of size 33,278. We use the default hyperparameters and train using stochastic gradient descent with an initial learning rate of 20. The learning rate is annealed 4x after each epoch if performance does not improve on the validation set. Results are shown in Table~\ref{table:lstm}.

\end{document}